\documentclass{article}

\PassOptionsToPackage{numbers, compress}{natbib}
\usepackage[preprint]{neurips_2021}

\usepackage[utf8]{inputenc} % allow utf-8 input
\usepackage[T1]{fontenc}    % use 8-bit T1 fonts
\usepackage{hyperref}       % hyperlinks
\usepackage{url}            % simple URL typesetting
\usepackage{booktabs}       % professional-quality tables
\usepackage{amsfonts}       % blackboard math symbols
\usepackage{nicefrac}       % compact symbols for 1/2, etc.
\usepackage{microtype}      % microtypography
\usepackage{xcolor}         % colors
\usepackage{fancyhdr}

\usepackage{amsthm}

\usepackage{graphicx}
\newtheorem{theorem}{Theorem}
\newtheorem{definition}{Definition}

\newtheorem{lemma}{Lemma}
\newtheorem{proposition}{Proposition}

\newtheorem{assumption}{Assumption}
\usepackage{bm}

\usepackage{nccmath}
\usepackage[labelformat=simple]{subcaption}

\usepackage{amsmath,cases}
\usepackage{algorithmic}
\usepackage{algorithm}

\newcommand{\req}[1]{Eq.~(\ref{#1})}
\newcommand{\rfig}[1]{Fig.~\ref{#1}}
\newcommand{\rtab}[1]{Tab.~\ref{#1}}

\title{Smoothness Analysis of Adversarial Training}
\author{%
  Sekitoshi Kanai\\
  NTT\\
  \texttt{sekitoshi.kanai.fu@hco.ntt.co.jp}
  \And
  Masanori Yamada\\
  NTT\\
  \texttt{masanori.yamada.cm@hco.ntt.co.jp} 
  \AND
  Hiroshi Takahashi \\
  NTT \\
  \texttt{hiroshi.takahashi.bm@hco.ntt.co.jp} 
  \And
  Yuuki Yamanaka\\
  NTT \\
  \texttt{yuuki.yamanaka.kb@hco.ntt.co.jp} 
  \And
  Yasutoshi Ida \\
  NTT \\
  \texttt{yasutoshi.ida@ieee.org} 
}

\begin{document}

\maketitle
\thispagestyle{fancy}
\lhead{\tiny The latest version of this article is published in IEEE Transactions on Neural Networks and Learning Systems (DOI: 10.1109/TNNLS.2023.3244172) \cite{IEEEkanai} under a Creative Commons License. \vspace{0.3cm}}

\begin{abstract}
    Deep neural networks are vulnerable to adversarial attacks.
    Recent studies about adversarial robustness focus on the loss landscape in the parameter space
    since it is related to optimization and generalization performance.
    These studies conclude that the difficulty of adversarial training is caused by
    the non-smoothness of the loss function: i.e., its gradient is not Lipschitz continuous.
    However, this analysis ignores the dependence of adversarial attacks on model parameters.
    Since adversarial attacks are optimized for models, they should depend on the parameters. 
    Considering this dependence, we analyze the smoothness of the loss function of adversarial training using the optimal attacks for the model parameter
    in more detail.
    We reveal that the constraint of adversarial attacks is one cause of the non-smoothness and that
    the smoothness depends on the types of the constraints.
    Specifically, the $L_\infty$ constraint can cause non-smoothness more than the $L_2$ constraint.
    Moreover, our analysis implies that if we flatten the loss function with respect to input data,
   the Lipschitz constant of the gradient of adversarial loss tends to increase.
    To address the non-smoothness, we show that EntropySGD smoothens the non-smooth loss
    and improves the performance of adversarial training.
\end{abstract}

\section{Introduction}
While deep learning is starting to play a crucial role in modern data analysis applications, 
deep learning applications are threatened by adversarial examples. 
Adversarial examples are data perturbed to make models misclassify them.
To improve the robustness against adversarial examples, a lot of studies have explored and presented
defense methods~\cite{carmon2019unlabeled,pgd,pgd2,zhang2019you,Wang2020Improving,distillation,cohen2019certified}. 
Especially, adversarial training has attracted the most attention~\cite{carmon2019unlabeled,pgd,pgd2}.
Adversarial training generates adversarial examples of training data by maximizing the loss function with norm constraints 
and minimizes the loss function on these adversarial examples with respect to parameters to obtain the robustness.

However, adversarial training is still difficult to achieve good accuracies on adversarial examples (hereinafter, called robust accuracies) 
compared with accuracies on clean data achieved by standard training.
To tackle the issue, several studies investigate 
the loss landscape in the parameter space~\cite{liu2020loss,awp,YamadaFlat}. 
\citet{awp} have experimentally revealed that loss landscapes in the parameter space of adversarial training 
can be sharp, which is a cause of the poor generalization performance of adversarial training. 
As a theoretical analysis, \citet{liu2020loss} have shown that the loss function of the adversarial training (hereinafter, called adversarial loss) is not a Lipschitz-smooth function: i.e., 
its gradient is not Lipchitz continuous~\cite{zhou2020regret}.
Since the gradient-based optimization is not very effective for non-smooth objective functions~\cite{NLProgram},
this indicates that minimization in adversarial training is more difficult than that in the standard training. 
However, this theoretical analysis ignores the dependence of adversarial examples on the model parameters.
Since adversarial examples are optimized for the models, they should depend on the parameters.

Considering the dependence,
we theoretically investigate the smoothness of adversarial loss in the parameter space and reveal
that the constraint of adversarial attacks is one cause of non-smoothness.
In addition, since our analysis still indicates that adversarial loss can be a non-smooth function,
we show that EntropySGD addresses the non-smoothness of adversarial loss.
First, we analyze the Lipschitz continuity of gradients of adversarial loss for binary linear classification
because its optimal adversarial examples are obtained in closed form.
As a result, we reveal that adversarial loss can be a locally smooth function
and that the smoothness depends on the constraints of adversarial examples.
Next, we extend the analysis to the general case using the optimal adversarial attack.
By using the condition of the local maximum point, we reveal that the adversarial loss using the optimal attack is a locally smooth function
if the optimal attack is inside the feasible region of the constraints (i.e., if the constraints do not affect the optimal attack).
In addition, even if the optimal attack is on the boundary of the feasible region,
the adversarial loss with the $L_2$ norm constraint becomes locally smooth.
Our analysis also implies that if we flatten the loss function with respect to input data, 
the Lipschitz constant of gradient with respect to parameters can increase: the Lipschitz constant depends on the singular values of Hessian matrices with respect to input data.
Since Lipschitz constants of gradient determine the convergence and stability of training~\cite{hardt2016train}, 
this result might explain why adversarial training is more difficult than standard training.
Finally, to show that the improvements in smoothness contributes adversarial training, we apply EntropySGD to adversarial training.
We prove that EntropySGD smoothens the non-smooth and non-negative loss function, 
and our experimental results demonstrate that EntropySGD accelerates adversarial training and improves robust accuracy.
\section{Preliminaries}
\subsection{Adversarial training and characteristics of adversarial loss}
An adversarial example $\bm{x}^\prime$ for a data point $\bm{x}\in \mathbb{R}^d$ with a label $y$ is formulated as
\begin{align}
    \textstyle \bm{x}'\textstyle =\bm{x}+\bm{\delta},~~
    \textstyle \bm{\delta}\textstyle =\mathrm{arg}\!\max_{||\bm{\delta}||_p\leq \varepsilon}\ell(\bm{x}+\bm{\delta},y,\bm{\theta}),\label{adex}
 \end{align}
where $||\cdot||_p$ is $L_p$ norm,
$\bm{\theta}\!\in\! \mathbb{R}^m$ is a parameter vector, $\varepsilon$ is a magnitude of adversarial examples, and $\ell$ is a loss function.\footnote{To save space, we use $||\cdot||$ for $L_2$ norm (Euclidean norm), and $\ell(\bm{x},y,\bm{\theta})$ is written as just $\ell(\bm{x},\bm{\theta})$ or $\ell(\bm{\theta})$. }
To obtain a robust model, adversarial training attempts to solve 
\begin{align}
    \textstyle\!\!\!\min_{\bm{\theta}}\!L_{\varepsilon}(\bm{\theta})\textstyle\!=\!\min_{\bm{\theta}}\frac{1}{N}\sum_{n=1}^{N}\ell_{\varepsilon} (\bm{x}_n,y_n,\bm{\theta})
    \textstyle\!=\!\min_{\bm{\theta}}\frac{1}{N}\sum_{n=1}^{N}\max_{||\bm{\delta}||_p\leq \varepsilon}\!\ell(\bm{x}_n\!+\!\bm{\delta},y_n,\bm{\theta}).\nonumber
\end{align}
To solve the inner maximum problem, projected gradient descent (PGD)~\cite{pgd,pgd2} is widely used.
For example, PGD with the $L_\infty$ constraint iteratively updates the adversarial examples as
\begin{align}
    \bm{\delta } \leftarrow \Pi_{\varepsilon}\left(\bm{\delta}+\eta_{\mathrm{P}}\mathrm{sign}\left(\nabla_{\bm{\delta}} \ell\left(\bm{x}+\bm{\delta}, y, \bm{\theta}\right)\right)\right),\nonumber
\end{align}
where $\eta_{\mathrm{P}}$ is a step size. $\Pi_{\varepsilon}$ clips elements of $\bm{\delta}$ into the range $[-\varepsilon, \varepsilon]$
as a projection operation into the feasible region $\{\bm{\theta}~|~||\bm{\delta}||_\infty\leq \varepsilon\}$.

We discuss Lipschitz-smoothness of adversarial loss: i.e., the Lipschitz continuity of the gradient of adversarial loss.
We use the following definition:
\begin{definition}
    $f(\bm{\theta})$ is $C_l$-Lipschitz on a set $\Theta$
    if there is a constant $0\!\leq\! C_l\!<\!\infty$ satisfying 
    \begin{align}\textstyle
        ||f(\bm{\theta}_1)\!-\!f(\bm{\theta}_2)||\!\leq\! C_l||\bm{\theta}_1\!-\!\bm{\theta}_2||,~~\forall \bm{\theta}_1,\bm{\theta}_2\in \Theta.
    \end{align}
    In addition, $f(\bm{\theta})$ is $C_s$-smooth on a $\Theta$ 
    if there is a constant $0\!\leq\! C_s\!<\!\infty$ satisfying 
    \begin{align}\textstyle\label{LipSmo}
        ||\nabla_{\theta} f(\bm{\theta}_1)\!-\!\nabla_{\theta} f(\bm{\theta}_2)||\!\leq\! C_s||\bm{\theta}_1\!-\!\bm{\theta}_2||,~~\forall \bm{\theta}_1,\bm{\theta}_2\in \Theta.
    \end{align}
\end{definition}
This smoothness of the objective function (\req{LipSmo}) is an important property for the convergence of gradient-based optimization~\cite{NLProgram}.
If we choose the appropriate learning rate (step size), a gradient method converges to a stationary point of the loss function $L(\bm{\theta})$
under the following condition~\cite{NLProgram}:
\begin{align}\label{LipCond}
    ||\nabla_{\theta} f(\bm{\theta}_1)\!-\!\nabla_{\theta} f(\bm{\theta}_2)||\!\leq\! C_s||\bm{\theta}_1\!-\!\bm{\theta}_2||,~~\forall \bm{\theta}_1,\bm{\theta}_2\in \{\bm{\theta}|L(\bm{\theta})\leq L(\bm{\theta}^0)\},
\end{align}
where $\bm{\theta}^0$ is the initial parameter of training, and we assume that the set $\{\bm{\theta}|L(\bm{\theta})\!\leq\! L(\bm{\theta}^0)\}$ is bounded.

\citet{liu2020loss} have analyzed the smoothness of adversarial loss under the following assumption:
\begin{assumption}
    \label{assp}
    For loss function $\ell(\bm{x},\bm{\theta})$, we have the following inequalities:
    \begin{align}
        \textstyle||\ell(\bm{x},\bm{\theta}_1)-\ell(\bm{x},\bm{\theta}_2)||\textstyle &\leq C_{\bm{\theta}}||\bm{\theta}_1-\bm{\theta}_2||,\\
        \textstyle ||\nabla_{\bm{\theta}}\ell(\bm{x},\bm{\theta}_1)-\nabla_{\bm{\theta}}\ell(\bm{x},\bm{\theta}_2)||&\textstyle \leq C_{\bm{\theta}\bm{\theta}}||\bm{\theta}_1-\bm{\theta}_2||,\label{orgC}\\
        \textstyle ||\nabla_{\bm{\theta}}\ell(\bm{x}_1,\bm{\theta})-\nabla_{\bm{\theta}}\ell(\bm{x}_2,\bm{\theta})||&\textstyle \leq C_{\bm{\theta}\bm{x}}||\bm{x}_1-\bm{x}_2||,
    \end{align}
    where $0\leq C_{\bm{\theta}}<\infty$, $0\leq C_{\bm{\theta}\bm{\theta}}<\infty$, and $0\leq C_{\bm{\theta}\bm{x}}<\infty$.
\end{assumption}
Assumption~\ref{assp} states that $\ell(\bm{x},\bm{\theta})$ is $C_{\bm{\theta}}$-Lipschitz, 
$C_{\bm{\theta}\bm{\theta}}$-smooth with respect to $\bm{\theta}$, and
$C_{\bm{\theta}\bm{x}}$-smooth with respect to $\bm{x}$.
Under Assumption~\ref{assp}, \citet{liu2020loss} have proven the following proposition:
\begin{proposition}{\cite{liu2020loss}}\label{liuProp}
    If the Assumption~\ref{assp} holds, we have
    \begin{align}
        \textstyle\!\!||L_{\varepsilon}(\bm{\theta}_1)\!-\!L_{\varepsilon}(\bm{\theta}_2)||&\!\textstyle\leq\! C_{\bm{\theta}}||\bm{\theta}_1-\bm{\theta}_2||,\\
        \textstyle \!\!||\nabla_{\bm{\theta}}L_{\varepsilon}(\bm{\theta}_1)\!-\!\nabla_{\bm{\theta}}L_{\varepsilon}(\bm{\theta}_2)||&\!\textstyle\leq \!C_{\bm{\theta}\bm{\theta}}||\bm{\theta}_1\!-\!\bm{\theta}_2||\!+\!2\varepsilon C_{\bm{\theta}\bm{x}}.\label{advsmo}
    \end{align}
\end{proposition}
From this proposition, \citet{liu2020loss} concluded that the adversarial loss $L_{\varepsilon}(\bm{\theta})$ is not globally smooth
for $\bm{\theta }$ because \req{advsmo} has a constant term.
If the gradient is not Lipschitz continuous, the gradient method is not effective as mentioned above.
However, if there is a set where the gradient of the loss function is locally Lipschitz continuous,
the gradient method can be effective under the condition of \req{LipCond}.
Therefore, this paper investigates the smoothness of adversarial loss in detail. 

\subsection{EntropySGD}
Many studies of deep learning for standard training
 have investigated loss landscape in parameter space because
a flat loss landscape improves the generalization performance~\cite{chaudhari2019entropy,sung2020s,wu2019noisy,jastrzkebski2017three,SHAM}.
To obtain a parameter vector $\bm{\theta}$ 
located on the flat landscape of the objective function $L(\bm{\theta})$, \citet{chaudhari2019entropy} have presented EntropySGD.
The objective function of EntropySGD is 
\begin{align}
    \textstyle\!\!\!\!-F(\bm{\theta})\!=-\mathrm{log}\!\int_{\bm{\theta}'}\!\mathrm{exp} \left(-\beta L(\bm{\theta}')-\beta\frac{\gamma}{2}||\bm{\theta}\!-\!\bm{\theta}'||_2^2 \right)\!d\bm{\theta}',
    \label{Opt}
\end{align}
where $\gamma$ is a hyperparameter, which determines a smoothness of loss. $\beta$ is usually set to one.
The gradient of EntropySGD with $\beta=1$ becomes 
\begin{align}\textstyle
    -\nabla_{\theta} F(\bm{\theta})=\gamma(\bm{\theta}-\mathbb{E}_{p_{\bm{\theta}}(\bm{\theta}')}[\bm{\theta}']),\label{GradEnSGD}
\end{align}
where $p_{\bm{\theta}}$ is a probability density function as
\begin{align}\textstyle
    p_{\bm{\theta}}(\bm{\theta}')=\frac{1}{\mathrm{exp}(F(\bm{\theta},\gamma))}\mathrm{exp} \left(- L(\bm{\theta}')-\frac{\gamma}{2}||\bm{\theta}\!-\!\bm{\theta}'||_2^2 \right).
    \label{PropEnSGD}
\end{align}
\citet{chaudhari2019entropy} have shown that EntropySGD improves the smoothness of the smooth loss: i.e., 
the Lipschitz constant of $-\nabla_{\theta} F(\bm{\theta})$ is smaller than that of $\nabla_{\theta} L(\bm{\theta})$
when $L(\bm{\theta})$ is smooth. 
However, for non-smooth loss functions, they do not show the effectiveness of EntropySGD.
Thus, the effectiveness of EntropySGD for adversarial training is still unclear due to non-smoothness.
\subsection{Related work}\label{rworkad}
    For improving adversarial robustness, several studies focus on the flatness of the loss landscape in the input data space
    because adversarial examples are perturbations in the input space~\cite{LLR,lmt,parseval,engstrom2018evaluating}.
    \citet{LLR} have presented a regularization method to flatten the loss landscape with respect to data points.
    Lipschitz constraints of the models can also be regarded as methods for flattening the loss in data space~\cite{lmt,parseval}.
    Few studies investigate the loss landscape of adversarial training in parameter space 
    because the relationship between the loss landscape in parameter space and the robustness had not been known before studies of~\cite{liu2020loss,awp}.
    \citet{awp} have experimentally shown that adversarial training can sharpen the loss landscape, which degrades generalization performance, and have proposed adversarial weight perturbation (AWP)
    to flatten the adversarial loss.
    Similarly, \citet{YamadaFlat} have investigated the flatness of adversarial loss for logistic regression. 
    Compared with these studies, we theoretically investigate the smoothness of adversarial loss with respect to parameters following~\cite{liu2020loss}.
    
    For standard training of deep neural networks (DNNs),
    \citet{keskar2016large} have shown that large batch training causes the sharp loss landscape, which
    causes the poor generalization performance. 
    \citet{Explor} have revealed the relationship between generalization performance and the flatness by using PAC-Bayes.
    A lot of studies have investigated approaches for flattening the loss landscape~\cite{SHAM,chaudhari2019entropy,sung2020s,wu2019noisy}.
    Since the spectral norm of the Hessian matrix is used as a measure of flatness~\cite{keskar2016large,chaudhari2019entropy,zhang2021flatness},
    there is a relationship between Lipschitz-smoothness and flatness:
    if a $C_s$-Lipschitz gradient $\nabla_{\theta} L(\bm{\theta})$ is everywhere differentiable,
    we have 
   \begin{align}\textstyle
       \textstyle\sup_{\bm{\theta}} \sigma_1(\nabla_{\theta}^2 L(\bm{\theta}))&\textstyle\leq C_s,\label{HesUp}
   \end{align}
   where $\sigma_i$ is the $i$-th largest singular value, and thus, $\sigma_1(\cdot)$ is a spectral norm.
   Equation~(\ref{HesUp}) indicates that smooth functions with small $C_s$ tend to have small spectral norms of the Hessian matrices, i.e., flat loss landscapes.   
We further discuss the relationship between flatness and smoothness in the supplementary materials.
We focus on EntropySGD since it is a basic method for smoothing the loss.
\section{Smoothness of adversarial training}
\label{smoanalysis}
\citet{liu2020loss} have investigated the smoothness of the adversarial loss as Proposition~\ref{liuProp}.
Its proof of sketch is as follows:
Let $\bm{x}^{1\prime}$ and $\bm{x}^{2\prime}$ be adversarial examples for $\bm{\theta}_1$ and $\bm{\theta}_2$, respectively.
We have
\begin{align}
    \!\!\!\textstyle||\nabla_{\bm{\theta}}\ell_{\varepsilon}(\bm{x},\bm{\theta}_1)-\nabla_{\bm{\theta}}\ell_{\varepsilon}(\bm{x},\bm{\theta}_2)||&=
    ||\nabla_{\bm{\theta}}\ell(\bm{x}^{1\prime},\bm{\theta}_1)\!-\!\nabla_{\bm{\theta}}\ell(\bm{x}^{2\prime},\bm{\theta}_2)||\nonumber\\
    &\leq 
    ||\nabla_{\bm{\theta}}\ell(\bm{x}^{1\prime},\bm{\theta}_1)\!-\!\nabla_{\bm{\theta}}\ell(\bm{x}^{1\prime},\bm{\theta}_2)||\!+\!||\nabla_{\bm{\theta}}\ell(\bm{x}^{1\prime},\bm{\theta}_2)\!-\!\nabla_{\bm{\theta}}\ell(\bm{x}^{2\prime},\bm{\theta}_2)||\nonumber
    \\&\leq C_{\bm{\theta}\bm{\theta}}||\bm{\theta}_1-\bm{\theta}_2||+C_{\bm{\theta}\bm{x}}||\bm{x}^{1\prime}-\bm{x}^{2\prime}||.\label{liuEq}    
\end{align}
Since $||\bm{x}-\bm{x}^\prime||\leq \varepsilon$ and $\nabla_{\theta} L_{\varepsilon}(\bm{\theta})$ is $\frac{1}{N}\!\sum_{n}\!\nabla_{\theta}  \ell_{\varepsilon}(\bm{x}_n, \bm{\theta})$, 
we can show $||\nabla_{\bm{\theta}}L_{\varepsilon}(\bm{\theta}_1)\!-\!\nabla_{\bm{\theta}}L_{\varepsilon}(\bm{\theta}_2)||\leq C_{\bm{\theta}\bm{\theta}}||\bm{\theta}_1-\bm{\theta}_2||+2\varepsilon C_{\bm{\theta}\bm{x}}$ from \req{liuEq}.
From Proposition~\ref{liuProp}, \citet{liu2020loss} conclude that the gradient of the adversarial loss is not Lipschitz continuous.
However, they ignore the dependence of adversarial examples on parameters: $\bm{x}^{\prime}$ can be regarded as the function of $\bm{\theta}$ as
$\bm{x}^{\prime}(\bm{\theta})$, i.e., $\bm{x}^{1\prime}=\bm{x}^\prime (\bm{\theta}_1)$.
Using $\bm{x}^{\prime} (\bm{\theta})$,  we can immediately derive the following lemma from \req{liuEq} :
\begin{lemma}\label{AdvBase}
    If the adversarial example $\bm{x}^\prime (\bm{\theta})$ is a $C$-Lipschitz function, 
    the gradient of adversarial loss is $(C_{\bm{\theta}\bm{\theta}}+CC_{\bm{\theta}\bm{x}})$-Lipschitz, that is, adversarial loss is $(C_{\bm{\theta}\bm{\theta}}+CC_{\bm{\theta}\bm{x}})$-smooth.    
\end{lemma}
This lemma indicates that Proposition~\ref{liuProp} does not prove non-smoothness of adversarial loss if adversarial examples are functions of the parameter.
Therefore, we need to analyze the dependence of adversarial examples on parameters.
However, adversarial examples $\bm{x}^{\prime}$ for DNNs cannot be solved in closed form.
Therefore, we first tackle this problem by using linear binary classification problems and next
investigate the Lipschitz continuity for general models using the optimal adversarial examples.
\subsection{Binary linear classification}\label{BNC}
First, we investigate the smoothness of the adversarial loss of the following problem:\\
\textbf{Problem formulation for Sec.~\ref{BNC}} ~
We have a dataset $\{(\bm{x},y)_n\}_{n=1}^{N}$ where $\bm{x}\in\mathbb{R}^d$ is a data point, $y\in \{-1,1\}$ is a binary label, and $\bm{\theta}\in \mathbb{R}^d$ is a parameter vector.
Let $f(\bm{x},\bm{\theta})\!=\!\mathrm{sign}(\bm{\theta}^T\bm{x})$ be a model and $\bm{\delta}$ be an adversarial perturbation 
whose $L_p$ norm is constrained as $||\bm{\delta}||_p\leq \varepsilon$.
We train $f(\bm{x},\bm{\theta})$ by minimizing the following adversarial loss:
\begin{align}
    \textstyle
\!\!\!\!L_{\varepsilon}(\bm{\theta})&\textstyle \!=\!\frac{1}{N}\sum_n \ell_{\varepsilon}(\bm{x}_n,y_n,\bm{\theta}),\nonumber \\\textstyle 
\!\!\!\ell_{\varepsilon}(\bm{x}_n,y_n,\bm{\theta})&\textstyle \!=\!\max_{||\bm{\delta}_n||_p\leq \varepsilon}\mathrm{log}\left(1+\mathrm{exp}(-y_n\bm{\theta}^T (\bm{x}_n+\bm{\delta}_n))\!\right)\!.\nonumber
\end{align}

For this binary linear classification, we can investigate the relationship between $\bm{x}^{\prime}$ and $\bm{\theta}$
because we can solve the optimal adversarial examples $\bm{x}^{\prime}\!=\!\bm{x}+\bm{\delta}$ in closed form.
The following lemma is a result for adversarial training with $L_2$ norm constraints $||\bm{\delta}||_2\leq \varepsilon$:
\begin{lemma}
    \label{l2lem}
    Let $\bm{x}^\prime (\bm{\theta}_1)$ and $\bm{x}^\prime (\bm{\theta}_2)$ be adversarial examples around the data point $\bm{x}$ for $\bm{\theta}_1$ and $\bm{\theta}_2$, respectively.
    For adversarial training with $L_2$ norm constraints $||\bm{\delta}||_2\leq \varepsilon$, 
    if there exists a lower bound $\theta_{min}\in \mathbb{R}$ such as $||\bm{\theta}||_2\!\geq\! \theta_{\min}\!>\!0$, we have the following inequality:
    \begin{align}
     \textstyle   ||\bm{x}^\prime (\bm{\theta}_1)\!-\!\bm{x}^\prime (\bm{\theta}_2)||\!\leq\!\frac{\varepsilon }{\theta_{\min}}||\bm{\theta}_1\!-\!\bm{\theta}_2||.
    \end{align}
    Thus, adversarial examples are $ (\frac{\varepsilon }{\theta_{\min}})$-Lipschitz on a bounded set not including the origin $\bm{\theta}= \bm{0}$.
\end{lemma}
This lemma indicates that adversarial examples with $L_2$ constraints are Lipschitz continuous function of $\bm{\theta}$ in the set excluding the origin.
From Lemmas~\ref{AdvBase} and \ref{l2lem}, we can derive the following theorem:
\begin{theorem}\label{l2Thm}
    For adversarial training with $L_2$ norm constraints $||\bm{\delta}||_2\!\leq\! \varepsilon$,
    if we have $||\bm{\theta}||_2\!\geq\! \theta_{\min}\!>\!0$, the following inequality holds:
    \begin{align}\label{l2c}
        \textstyle  \!\!||\nabla_{\bm{\theta}}L_{\varepsilon}(\bm{\theta}_1)\!-\!\nabla_{\bm{\theta}}L_{\varepsilon}(\bm{\theta}_2)||\!\leq\!(C_{\bm{\theta}\bm{\theta}}\!+\!\frac{\varepsilon C_{\bm{\theta}\bm{x}}}{\theta_{\min}})||\bm{\theta}_1\!-\!\bm{\theta}_2||.
    \end{align}
    Thus, adversarial loss $L_{\varepsilon}$ is $ (C_{\bm{\theta}\bm{\theta}}\!+\!\frac{\varepsilon C_{\bm{\theta}\bm{x}}}{\theta_{\min}})$-smooth on a bounded set that does not include $\bm{\theta}\!=\!\bm{0}$.
\end{theorem}
This theorem shows that the adversarial loss for the binary linear classification with the $L_2$ constraint
is a smooth function for $\bm{\theta}\!\neq\! \bm{0}$.
From \req{LipCond}, if $L_{\varepsilon}(\bm{\theta}^0)\!<\! L_{\varepsilon}(\bm{0})$, a gradient method is effective in this case.
Note that the Lipschitz constant ($ C_{\bm{\theta}\bm{\theta}}\!+\!\frac{\varepsilon C_{\bm{\theta}\bm{x}}}{\theta_{\min}}$)
 is larger than that for standard training ($ C_{\bm{\theta}\bm{\theta}}$).

Next, we provide the same analysis for the $L_\infty$ constraint of adversarial examples as follows:
\begin{lemma}
    \label{linflem}
    For adversarial training with $L_\infty$ norm constraints $||\bm{\delta}||_\infty \!\leq\! \varepsilon$,
    if the sign of one element at least is different between $\bm{\theta}_1$ and $\bm{\theta}_2$ ($\exists i\!:\!\mathrm{sign}(\theta_{1,i})\neq \mathrm{sign}(\theta_{2,i})$),
    adversarial examples are not Lipschitz continuous.
    If all signs are the same ($\forall i\!:\!\mathrm{sign}(\theta_{1,i})= \mathrm{sign}(\theta_{2,i})$),
    we have 
    \begin{align}
        \textstyle  ||\bm{x}^\prime (\bm{\theta}_1)\!-\!\bm{x}^\prime (\bm{\theta}_2)||\!=\!0.
    \end{align}
    Thus, adversarial examples are Lipschitz continuous on a bounded set
    that does not include $\theta_i= 0,\forall i$ and where no signs of elements change.
\end{lemma}
This lemma indicates that the set where adversarial examples with $L_\infty$ norm constraints 
are Lipschitz continuous is smaller than the set where those with $L_2$ norm constraints 
are Lipschitz continuous.
From Lemmas~\ref{AdvBase} and \ref{linflem}, we have the following theorem:
\begin{theorem}\label{linfThm}
    For adversarial training with $L_\infty$ norm constraints $||\bm{\delta}||_\infty \!\leq\! \varepsilon$,
    in the set where $\forall i:\mathrm{sign}(\theta_{1,i})\!=\! \mathrm{sign}(\theta_{2,i})$,
    the following inequality holds:
    \begin{align}
        \textstyle    ||\nabla_{\bm{\theta}}L_{\varepsilon}(\bm{\theta}_1)-\nabla_{\bm{\theta}}L_{\varepsilon}(\bm{\theta}_2)||&\leq C_{\bm{\theta}\bm{\theta}}||\bm{\theta}_1-\bm{\theta}_2||.
    \end{align}
    Thus, the loss function $L_{\varepsilon}$ for adversarial training is $C_{\bm{\theta}\bm{\theta}}$-smooth
    on a bounded set that does not include $\theta_i= 0,\forall i$ and where no signs of elements change.
\end{theorem}
Figure~\ref{AdvEx} shows the intuition of Lemmas~\ref{l2lem} and \ref{linflem}.
In the case of the $L_2$ constraint, adversarial examples continuously move on the circle depending on $\bm{\theta}$.
On the other hand, in the case of the $L_\infty$ constraint,
 even if the difference between $\bm{\theta}_1$ and $\bm{\theta}_2$ is only small,
the distance between $\bm{x}^{1\prime}$ and $\bm{x}^{2\prime}$ can be $2\varepsilon$, which is the length of the side of the square.
This is because the adversarial examples are located at the corner of the square.
Thus, the Lipschitz continuity of adversarial examples depend on the types of the constraint for adversarial examples.
\begin{figure}[tb]
    \centering
    \begin{subfigure}[t]{0.37\linewidth}
    \includegraphics[width=1.2\linewidth]{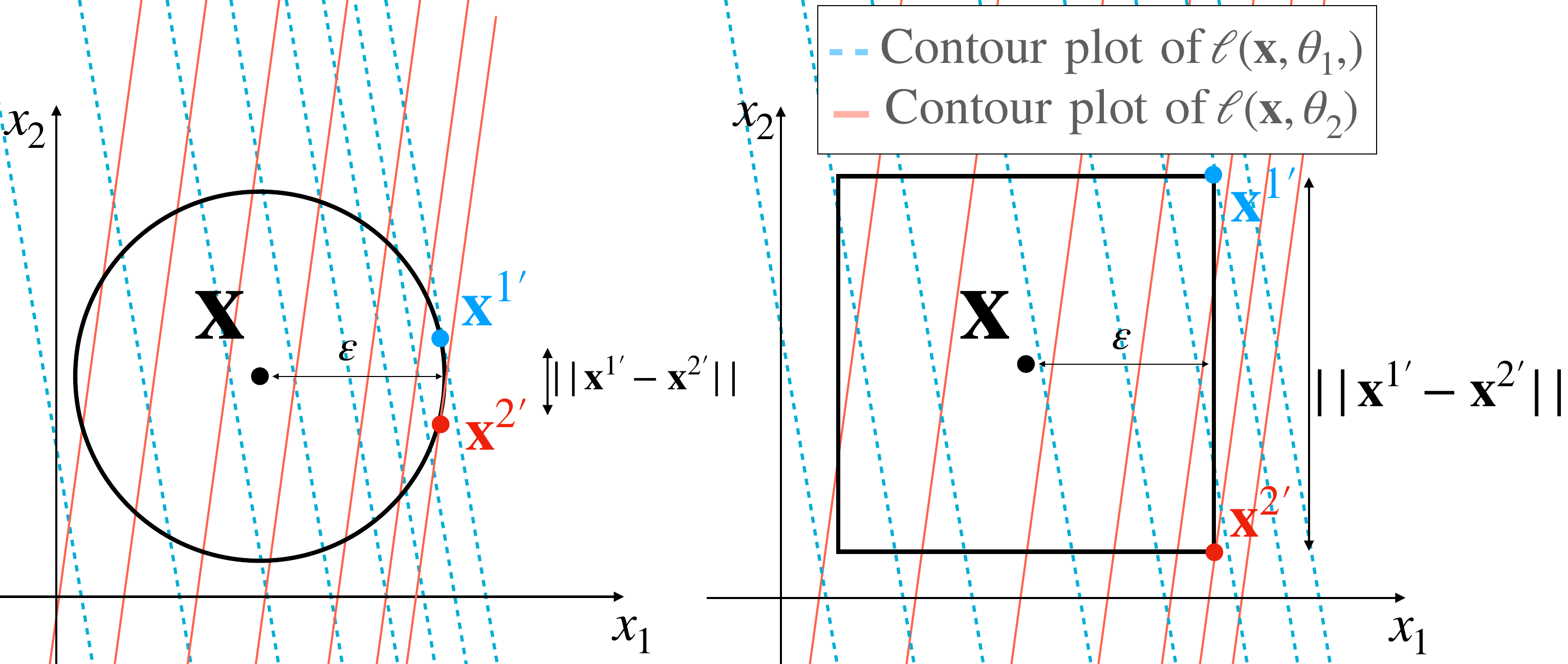}
    \caption{$\bm{x}'$ for the binary linear classification. 
    }\label{AdvEx}
\end{subfigure}
\hspace{\fill}
        \begin{subfigure}[t]{0.3\linewidth}
        \centering
        \hspace{5pt}
        \includegraphics[width=.65\linewidth]{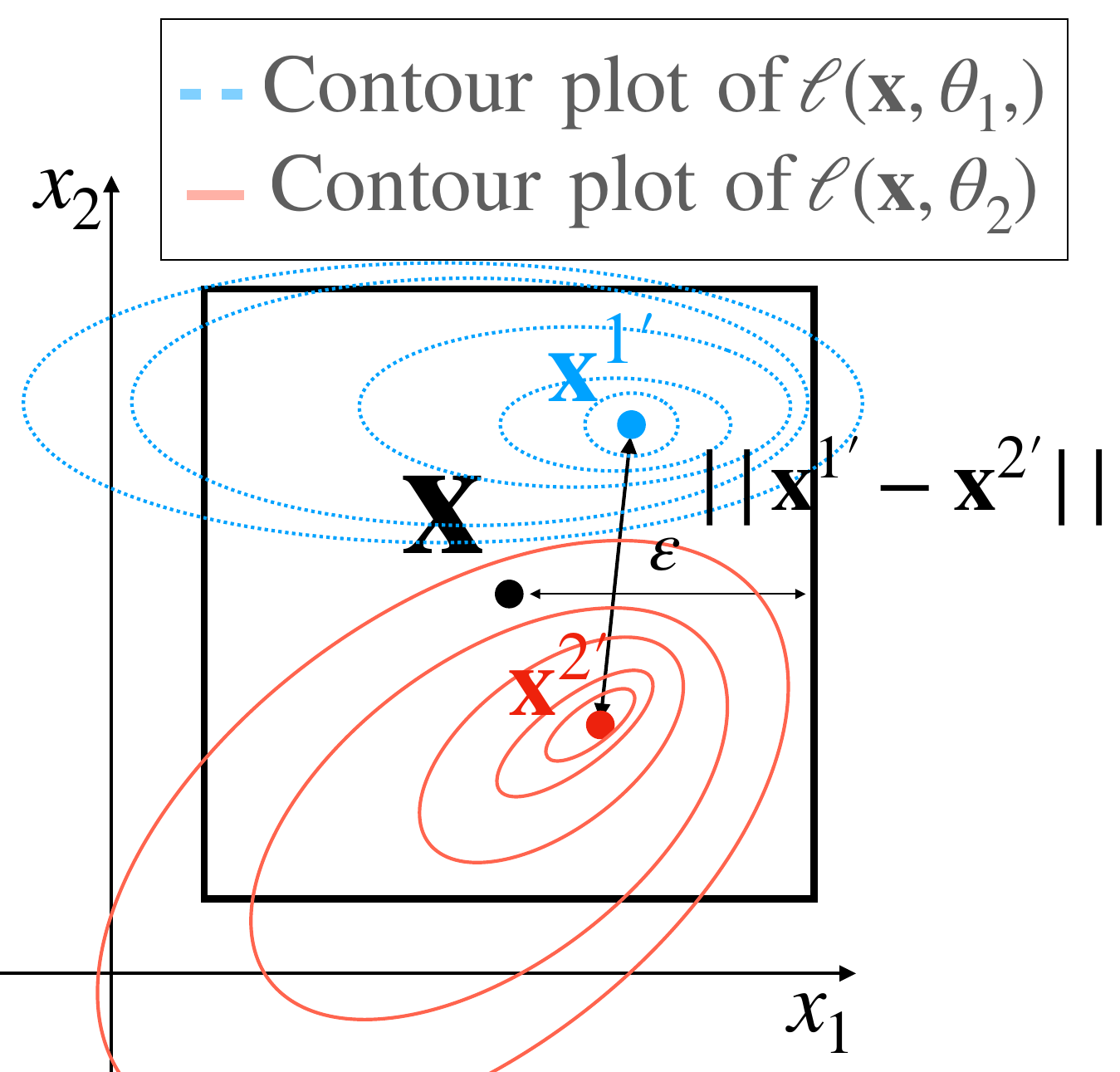}
        \caption{$\bm{x}'$ with the inactive constraint.}
        \label{InSi}
        \end{subfigure}
        \hspace{\fill}
        \begin{subfigure}[t]{0.3\linewidth}
        \centering
        \includegraphics[width=.63\linewidth]{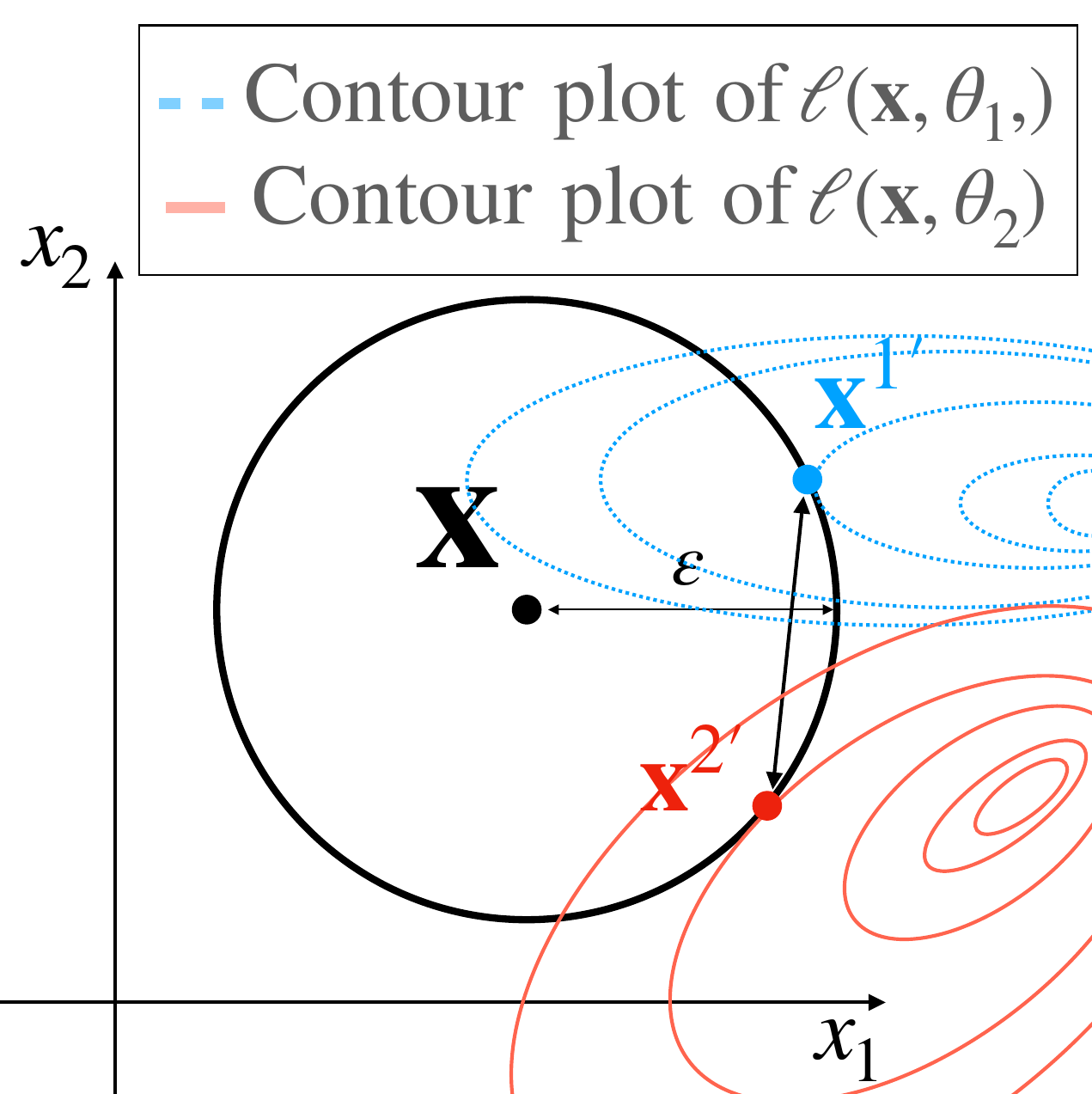}
        \caption{$\bm{x}'$ with the active constraint.}
        \label{OnEd}
        \end{subfigure}
        \caption{Illustrations of adversarial examples for $\bm{x}\!\in\! \mathbb{R}^2$.
        A circle and square correspond to the feasible regions for $L_2$ and $L_\infty$ constraints, respectively.    
        (a) shows adversarial examples for binary linear classification. 
        (b) and (c) show the optimal adversarial examples inside and on the boundary of the feasible region without assumptions about models, respectively. 
        }
\end{figure}
\subsection{General case}
For DNNs and multi-class classification, we could not obtain adversarial examples in closed form.
For this case, 
we investigate the local Lipchitz smoothness of adversarial loss with the local optimal adversarial examples by using the implicit function theorem~\cite{NLProgram}.
\subsubsection{Adversarial examples inside the feasible region}
The optimal adversarial examples for DNNs can be inside the feasible region $\{\bm{\delta}|~||\bm{\delta}||_p\!<\!\varepsilon\}$ (\rfig{InSi})~\cite{kim2020understanding},
while those for linear models are on the boundary of the feasible region (\rfig{AdvEx}).
Since the constraint is inactive in this case, the optimal adversarial examples $\bm{x}^{\prime*}$ satisfy $\nabla_x \ell (\bm{x}^{\prime*},\bm{\theta})\!=\!\bm{0}$,
and we show local Lipschitz continuity of $\bm{x}^{\prime}$ by applying the implicit function theorem to $\nabla_x \ell (\bm{x}^{\prime*},\bm{\theta})\!=\!\bm{0}$:
\begin{lemma}
    \label{InSLem}
    We assume that $\nabla_x \ell(\bm{x},\bm{\theta})$ is a $C^1$ function and $\bm{x}^{\prime*}$ is the local maximum point satisfying $\nabla^2_x \ell (\bm{x}^{\prime*},\bm{\theta})\!\prec\! 0$ inside the feasible region ($||\bm{x}^{\prime*}-\bm{x}||_p\!<\!\varepsilon$).\footnote{$\bm{A}\!\prec\! 0$ represents that $\bm{A}$ is negative definite.}
    If there is a constant $0\!<\!c\!<\!\infty$ such as $\max_i\lambda_i(\nabla^2_x \ell (\bm{x}^{\prime},\bm{\theta}))\!\leq\! -c$ where $\lambda_i$ is the $i$-th eigenvalue,
    the optimal adversarial example $\bm{x}^{\prime*}$ in some neighborhood $U$ of $\bm{\theta}$ is a continuously differentiable function $\bm{x}^\prime(\bm{\theta})$ and we have
    \begin{align}\textstyle
        ||\bm{x}^\prime(\bm{\theta}_1)-\bm{x}^\prime(\bm{\theta}_2)||\leq \frac{C_{\bm{\theta x}}}{c}||\bm{\theta}_1-\bm{\theta}_2||~~\forall\bm{\theta}_1,\bm{\theta}_2\in U. 
    \end{align}
\end{lemma}
From Lemmas~\ref{AdvBase} and \ref{InSLem}, adversarial loss can be locally $(C_{\bm{\theta\theta}}\!+\!C_{\bm{\theta x}}^2/c)$-smooth on the parameter set
such that the local maximum point $\bm{x}^{\prime *}$ exists inside the feasible region.
Note that $\nabla^2_x \ell (\bm{x}^{\prime*},\bm{\theta})\prec 0$ is a sufficient condition for the local maximum points, and thus, $\bm{x}^{\prime*}$ is expected to have a constant $c$ such as $\max_i\lambda_i(\nabla^2_x \ell (\bm{x}^{\prime*},\bm{\theta}))\leq -c$.
This lemma also indicates that non-smoothness of adversarial loss is caused by the constraints of adversarial examples:
if the constraints of adversarial examples are inactive, the adversarial loss tends to be smooth under Assumption~1.
Even so, the condition where the optimal adversarial examples exist inside the feasible region might be easily broken by the change of $\bm{\theta}$.
Next, we investigate the case when the constraints of adversarial attacks are active.
\subsubsection{Adversarial examples on the boundary of the feasible region}
In the previous section, we have shown the continuity of adversarial examples by applying the implicit function theorem to $\nabla_x \ell (\bm{x}^{\prime*},\bm{\theta})\!=\!\bm{0}$.
However, when the constraints of adversarial examples are active (\rfig{OnEd}), we cannot prove it in the same manner because $\nabla_x \ell (\bm{x}^{\prime*}\bm{\theta})\!\neq\! \bm{0}$.
In this case, instead of $\nabla_x \ell (\bm{x}^\prime,\bm{\theta})$, we use the gradient of the Lagrange function of adversarial examples for the implicit function theorem.
The Lagrange function $J(\bm{x}^\prime,\bm{\theta},\mu)$ with the $L_p$ norm constraint is given by
\begin{align}\textstyle
J(\bm{x}^\prime,\bm{\theta},\mu)=\ell (\bm{x}',\bm{\theta})-\mu (||\bm{x}'-\bm{x}||_p-\varepsilon),
\end{align}
where $\mu\!\in\! \mathbb{R}$ is a Lagrange multiplier. 
When we define $\tilde{\bm{x}}\!=\![\mu,\bm{x}^{\prime T}]^T$, 
the local maximum point $\bm{x}^{\prime*}$ of this problem satisfies
$\nabla_{\tilde{x}} J (\tilde{\bm{x}}^*,\bm{\theta})\!=\!\bm{0}$.
By applying the implicit function theorem to $\nabla_{\tilde{\bm{x}}}J(\bm{x}^\prime,\bm{\theta})\!=\!\bm{0}$,
we show the local continuity of $L_2$ attacks:
\begin{lemma}
\label{EdgeLem}
We assume that $\nabla_{\!x}\ell (\bm{x},\bm{\theta})$ is a $C^1$ function and $\tilde{\bm{x}}^{*}\!=\![\mu^*,\bm{x}^{\prime*T}]^T$ is the local maximum point satisfying $\mathrm{det}\!\left(\nabla^2_{\!\tilde{x}} J (\bm{\tilde{x}}^{*},\bm{\theta})\right)\!>\!0$
on the boundary of the feasible region of the $L_2$ constraint ($||\bm{x}^{\prime*}\!-\!\bm{x}||_2\!=\!\varepsilon$).
If there is a constant $0\!<\!c\!<\!\infty$ such as $\min_i\sigma_i(\nabla^2_{\!\tilde{x}} J (\tilde{\bm{x}},\bm{\theta}))\!\geq\! c$,
the local maximum point $\bm{x}^{\prime*}$ in some neighborhood $U$ of $\bm{\theta}$ is a continuously differentiable function $\bm{x}^\prime(\bm{\theta})$ and we have
\begin{align}\textstyle
    ||\bm{x}^\prime(\bm{\theta}_1)-\bm{x}^\prime(\bm{\theta}_2)||\leq \frac{C_{\theta x}}{c}||\bm{\theta}_1-\bm{\theta}_2||~~\forall\bm{\theta}_1,\bm{\theta}_2\in U.
\end{align}
\end{lemma}
Lemma~\ref{EdgeLem} uses the bordered Hessian matrix: 
\begin{align}\textstyle
\nabla^2_{\tilde{x}} J (\bm{\tilde{x}},\bm{\theta})=
\begin{bmatrix}\textstyle
    0&(\nabla_{x^\prime} ||\bm{x}'-\bm{x}||_p)^T\\
\nabla_{x^\prime} ||\bm{x}'-\bm{x}||_p&\nabla_{x^\prime}^2\ell (\bm{x}^\prime,\bm{\theta})-\mu\nabla_{x^\prime}^2||\bm{x}'-\bm{x}||_p
\end{bmatrix}.\label{BHes}
\end{align} 
To compute \req{BHes}, $||\bm{x}'-\bm{x}||_p$ should be twice differentiable, which is satisfied when $p\!=\!2$.
The condition $\mathrm{det}\left(\nabla^2_{\tilde{x}} J (\bm{\tilde{x}}^{*},\bm{\theta})\right)\!> \!0$ is a sufficient condition
of the local maximum point for constrained optimization problems~\cite{NLProgram}. 
From Lemmas~\ref{AdvBase} and \ref{EdgeLem},
adversarial loss can be locally $(C_{\bm{\theta\theta}}\!+\!C_{\bm{\theta x}}^2/c)$-smooth under the $L_2$ constraint even if the constraint is active.
On the other hand, if we use the $L_\infty$ constraint, we cannot show the same results 
because $||\bm{x}'\!-\!\bm{x}||_{\infty}$ is not twice differentiable.
Thus, it is difficult to show the Lipschitz continuity of attacks with the $L_{\infty}$ norm constraint.
This result also implies that non-smoothness of adversarial loss is caused by constraints.

Intriguingly,
Lemmas~\ref{InSLem} and \ref{EdgeLem} reveal the relationship between the flatness of the loss function with respect to input data
and the smoothness of the adversarial loss with respect to parameters.
If we flatten the loss function in the input space for robustness as described in Sec.~\ref{rworkad},
the singular values of the Hessian matrix $\nabla_x^2 \ell$ become small: i.e., $c$ can be a small value.
As a result, the Lipschitz constant of gradient $(C_{\bm{\theta\theta}}\!+\!C_{\bm{\theta x}}^2/c)$ increases,
and thus, the Lipschitz constant of the gradient of adversarial loss  with respect to $\bm{\theta}$ increases.
Since large Lipschitz constants decrease the convergence and stability of training~\cite{hardt2016train}, 
this relationship can explain why adversarial training is more difficult than standard training
even if adversarial loss is smooth.

\subsubsection{Limitations of the analysis}\label{limit}
From the above results, the adversarial loss can be locally smooth if
adversarial loss always uses the local optimal attack near the attack in the previous parameter update.
However, there might be several local maximum points of $\bm{x}'$ due to non-convexity of $\ell$ and adversarial loss $\ell_{\varepsilon}$
can use a different local maximum point for each parameter update. In this case, adversarial loss might be non-smooth even if we use the $L_2$ constraint.
In addition, the optimal attacks are difficult to find due to non-convexity, and we empirically use PGD attacks for generating adversarial examples.
We can conjecture that adversarial loss with the $L_\infty$ constraints using PGD does not have globally
Lipschitz continuous gradients because projection $\Pi_{\varepsilon}$ in PGD is not a continuous function.
Although non-singularity of Hessian matrices ($\nabla^2_x \ell (\bm{x},\bm{\theta})\!\prec\! 0$ and $\mathrm{det}\!\left(\nabla^2_{\tilde{x}} J (\bm{\tilde{x}},\bm{\theta})\right)\!>\!0$) is a sufficient condition for the local maximum point in Lemmas~\ref{InSLem} and \ref{EdgeLem}, it can be broken by the change in parameter~$\bm{\theta}$.
From the above limitations, the adversarial loss tends to be non-smooth more often than the clean loss especially using the $L_\infty$ constraint, 
and we should address the non-smoothness of the adversarial loss.
\section{EntropySGD for adversarial training}
In the previous section, we showed that adversarial training increases Lipschitz constants of gradient of adversarial loss and can cause non-smoothness.
If the loss function is non-smooth, the gradient-based optimization is not very effective.
To tackle this problem, we show that EntropySGD smoothens the non-smooth loss and can be used for adversarial training.
We prove the following theorem:
\begin{theorem}\label{EnThm}
    Let $\bm{\Sigma}_{\bm{\theta}^\prime}$ be a variance-covariance matrix of $p_{\bm{\theta}}(\bm{\theta}^\prime )$ in \req{PropEnSGD}.
    If we use EntropySGD for the non-negative loss function $L(\bm{\theta})\!\geq\! 0$, we have 
    \begin{align}\textstyle
        \!\!\!||\nabla_{\bm{\theta}} F(\bm{\theta}_1)\!-\!\nabla_{\bm{\theta}} F(\bm{\theta}_2))||\!\leq \!\left(\gamma \!+\! \gamma^2 \sup_{\bm{\theta}} ||\bm{\Sigma_{\bm{\theta}^\prime}}||_F \right )||\bm{\theta}_1\!-\!\bm{\theta}_2||, 
    \end{align} 
    and $||\bm{\Sigma}_{\bm{\theta}^\prime}||_{F}\!< \!\infty$ for $\bm{\theta}\!\in\! \mathbb{R}^m$.
    Thus, $-\nabla_{\bm{\theta}} F(\bm{\theta})$ is Lipschitz continuous on a bounded set of $\bm{\theta}$.
\end{theorem}
Theorem~\ref{EnThm} indicates that EntropySGD smoothens non-smooth loss functions. Many non-negative loss functions (e.g., cross-entropy for the classification) are used for training of DNNs.
Thus, we can use EntropySGD for adversarial training whose loss does not necessarily have Lipchitz continuous gradient.
Note that our analysis reveals that the Hessian matrix of EntropySGD is composed of the variance-covariance matrix of $p_{\bm{\theta}}(\bm{\theta}^\prime )$,
and we evaluate an extension of EntropySGD using this matrix in the supplementary materials.
\section{Experiments}
We first visualize the loss surface of adversarial loss and the loss of EntropySGD (EnSGD)
to verify our theoretical results: the smoothness depends on the types of the constraints, and EnSGD smoothens non-smooth functions.
Next, we demonstrate that improvements in smoothness by EnSGD contribute to the performance of adversarial training.
\subsection{Visualization of loss surface}
\begin{figure*}[tb]
    \begin{subfigure}{0.24\linewidth}\centering
        \includegraphics[width=.9\linewidth]{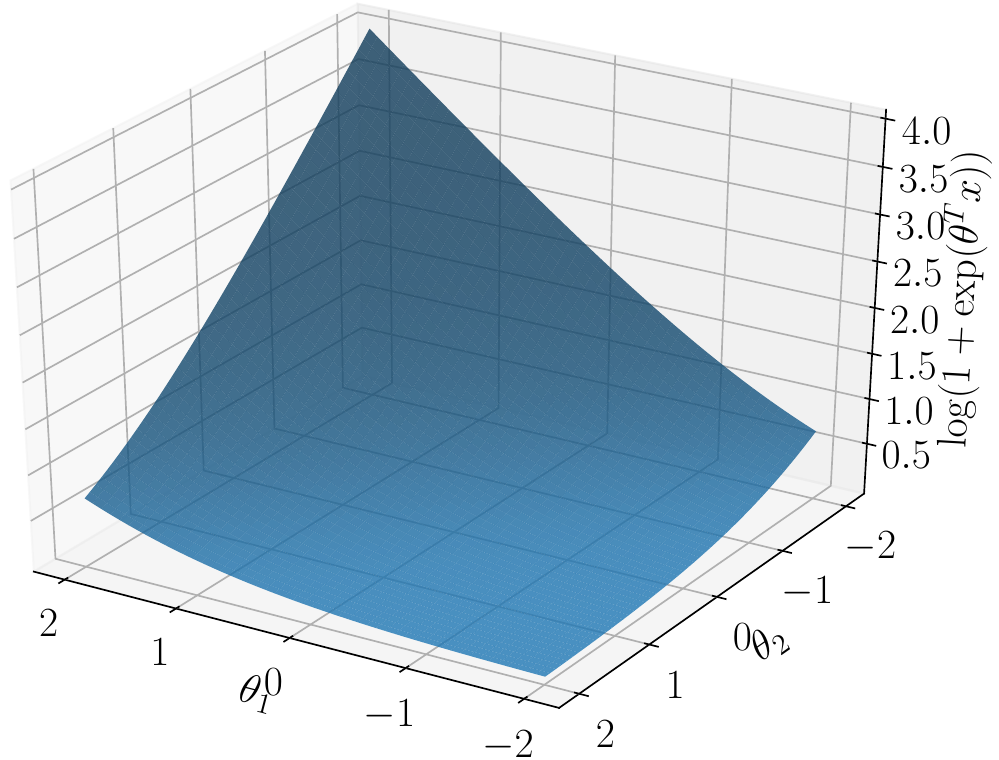}
        \caption{Clean loss}
    \end{subfigure}
    \begin{subfigure}{0.24\linewidth}\centering
        \includegraphics[width=.9\linewidth]{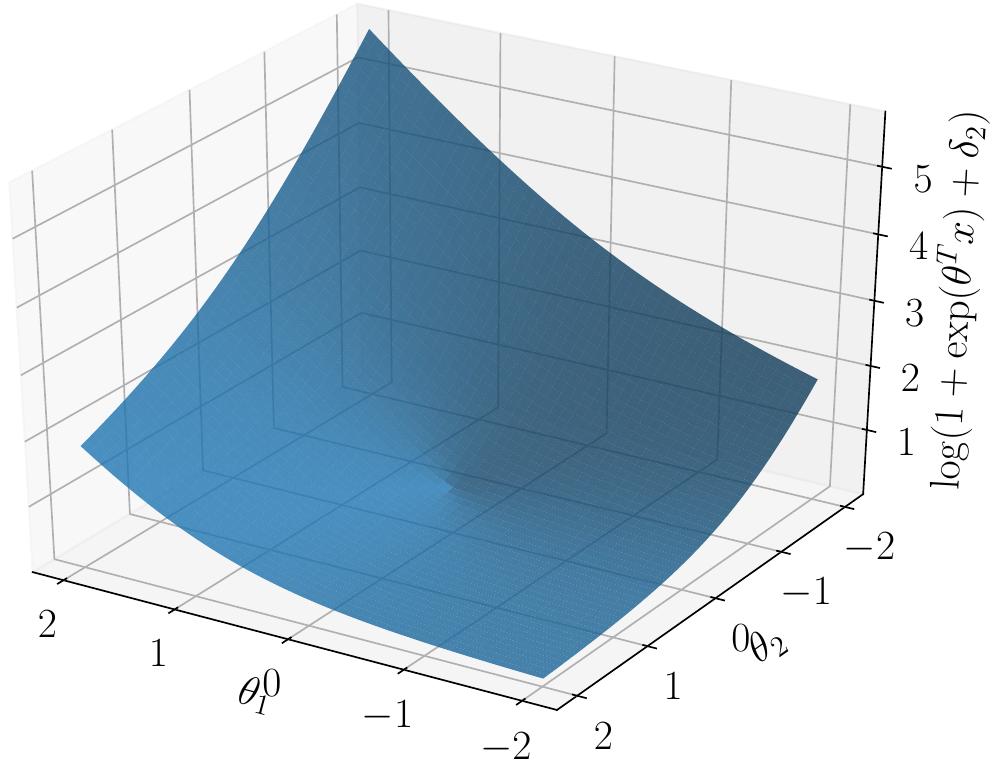}
        \caption{Adv. loss with $L_2$} \label{LinL2P}
    \end{subfigure}
    \begin{subfigure}{0.24\linewidth}\centering
            \includegraphics[width=.9\linewidth]{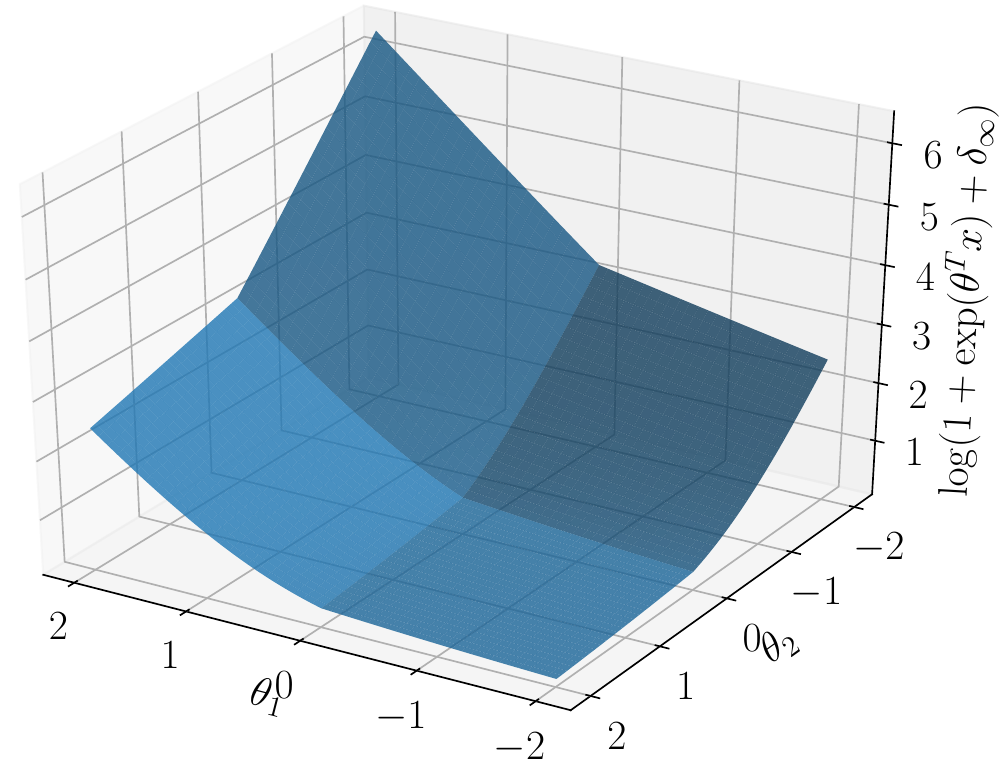}
            \caption{Adv. loss with $L_\infty$} \label{LinLinfP}
    \end{subfigure}    
    \hspace{\fill}
    \begin{subfigure}{0.24\linewidth}\centering
                    \includegraphics[width=.8\linewidth]{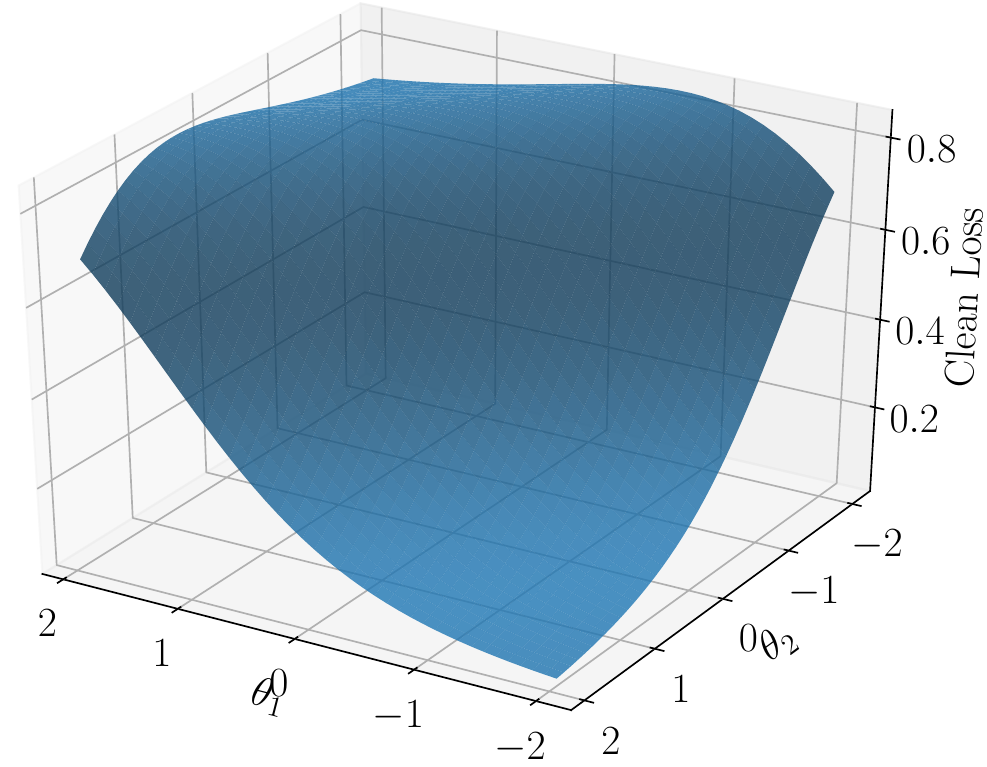}
                    \caption{Clean loss using swish}\label{SwishP}
                \end{subfigure}
                \hspace{\fill}
                \begin{subfigure}{0.24\linewidth}\centering
                    \includegraphics[width=.8\linewidth]{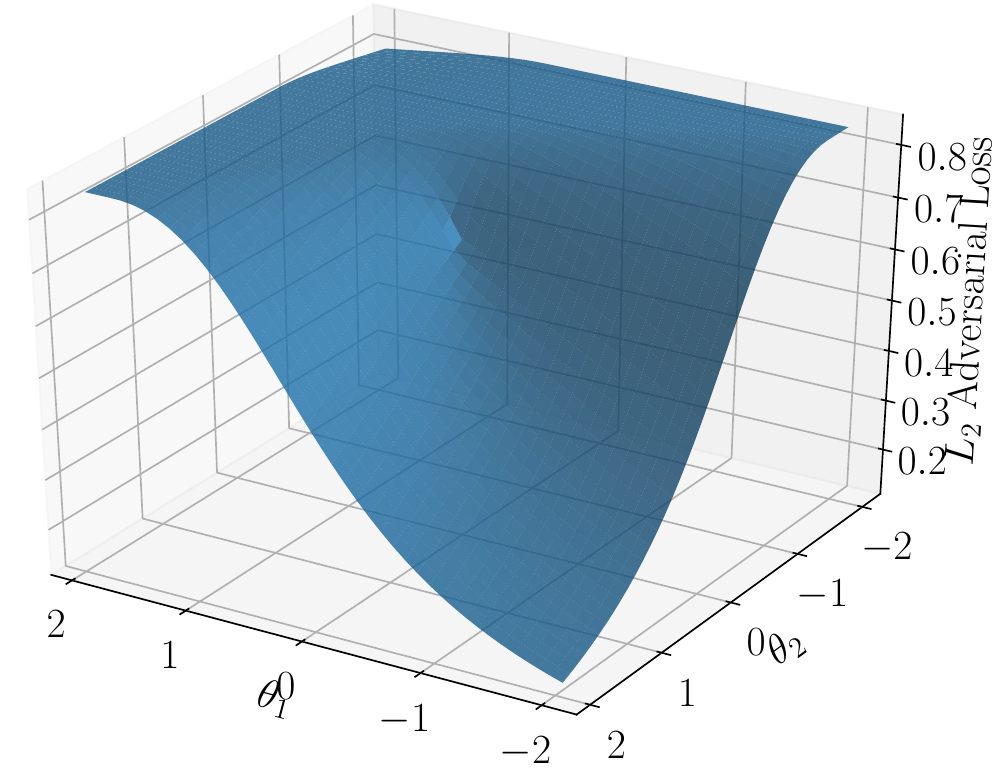}
                    \caption{Adv. loss using swish and the $L_2$ constraint} \label{SwishL2P}
                \end{subfigure}
                \hspace{\fill}
                \begin{subfigure}{0.236\linewidth}\centering
                    \includegraphics[width=.8\linewidth]{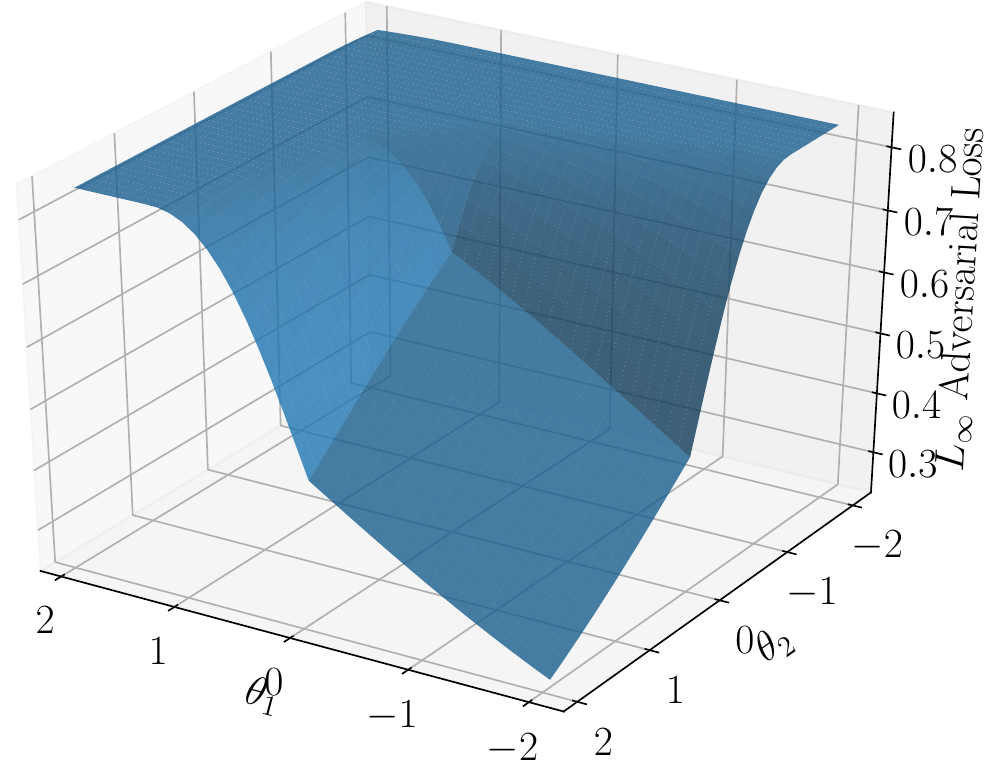}
                    \caption{Adv. loss using swish and the $L_\infty$ constraint} \label{SwishLinfP}
                \end{subfigure}
                \begin{subfigure}{0.25\linewidth}\centering
                        \includegraphics[width=.9\linewidth]{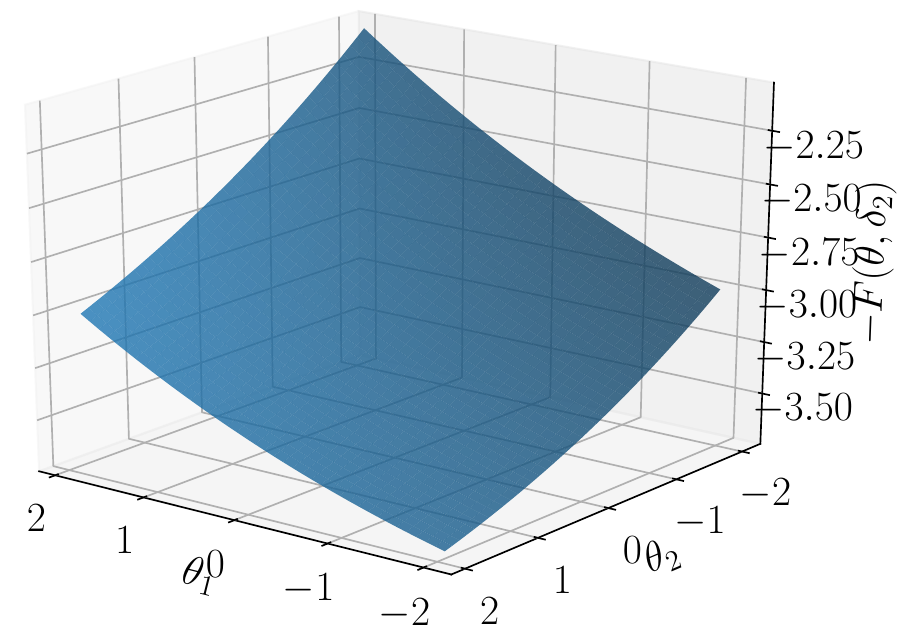}
                        \caption{EnSGD with $L_2$} \label{EnSGDL2P}
                    \end{subfigure}
                        \begin{subfigure}{0.25\linewidth}\centering
                                    \includegraphics[width=.9\linewidth]{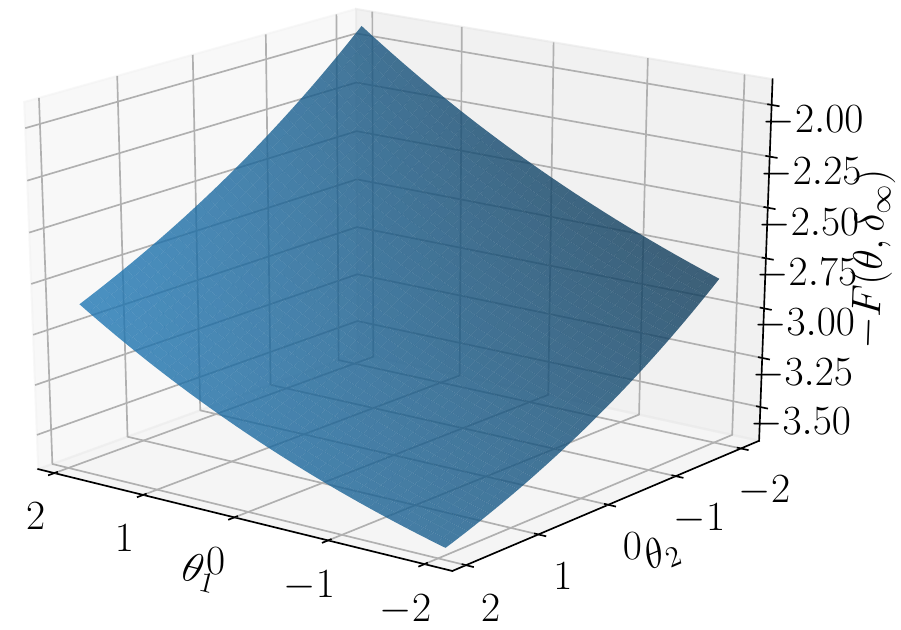}
                                    \caption{EnSGD with $L_\infty$} \label{EnSGDLinfP}
                        \end{subfigure}
                \caption{Loss surface for a linear model ((a)-(c)): $\ell_{\varepsilon}(\bm{\theta})\!=\!\log(1\!+\!\exp(-y\bm{\theta}^T(\bm{x}\!+\!\bm{\delta})))$
            and a nonlinear model ((d)-(f)): $\ell_{\varepsilon} (\bm{\theta})\!=\!\log(1\!+\!\exp(-y z))$ where $z\!=\!\mathrm{swish}(u)$, $u\!=\!\bm{\theta}^T(\bm{x}\!+\!\bm{\delta})$, $\bm{x}\!=\![-1,1]^T$, $y\!=\!1$, and 
    $\theta_{i}\in[-2,2]$, and $||\bm{\delta}||_p\leq 0.6$.
      (g) and (h) plot $-F(\bm{\theta})$ of EnSGD with $\ell_{\varepsilon}(\bm{\theta})$ for a linear model.}
    \label{Plot}
\end{figure*}
Figure~\ref{Plot} plots loss surfaces for one data point $\ell(\bm{\theta})$ in the standard training, adversarial training with the $L_2$ constraint, and
adversarial training with the $L_\infty$ constraint when $d$ is two. 
Adversarial losses of a linear model (Fig.~\ref{LinL2P} and \subref{LinLinfP}) follow Theorems~\ref{l2Thm} and \ref{linfThm}: adversarial loss with the $L_2$ constraint is not continuous at $\bm{\theta}\!=\!\bm{0}$
and  adversarial loss with the $L_\infty$ constraint is not continuous where $\theta_i\!=\!0$.

Figures~\ref{SwishP}-\subref{SwishLinfP} plot the loss surface for the binary classification using a nonlinear model that 
has swish~\cite{swish} before the output of the model.
Since we cannot compute the optimal adversarial examples in closed form, 
we used PGD attacks to generate adversarial examples.
The adversarial loss with the $L_2$ constraint (\rfig{SwishL2P}) have larger sets where the loss is smooth than that with the $L_\infty$
constraint (\rfig{SwishLinfP}), which follows Lemma~\ref{EdgeLem}.
In \rfig{SwishLinfP}, though we use the $L_\infty$ constraint, 
the adversarial loss can be smooth in the region where $\theta_1\!>\!1$ and $\theta_2\!<\!-1$
unlike the linear model (\rfig{LinLinfP}).
This is because the optimal attacks are inside the feasible region and satisfy Lemma~\ref{InSLem}: 
there exists $\bm{x}^{\prime*}$ satisfying $\nabla_x \ell(\bm{x}^{\prime*})\!=\!\frac{\partial \ell }{\partial z}\frac{\partial z}{\partial u}\nabla_x u(\bm{x}^{\prime*})\!=\!\bm{0}$ inside the region of $||\bm{x}-\bm{x}^{\prime}||_\infty\!\leq\! 0.6$
because the optimal input $u^*\!=\!\bm{\theta}^T\bm{x}^{\prime*}$ satisfying  $\frac{\partial z}{\partial u}\!=\!\mathrm{sigmoid}(u)+u\mathrm{sigmoid}(u)(1\!-\mathrm{sigmoid}(u))\!=\!0$ is in the interval of $[-2,-1]$.
Note that though we use a one-layer simple network to visualize the loss surface, Lemmas~\ref{InSLem} and \ref{EdgeLem} are not limited to shallow networks and binary classification problems.

Figures~\ref{EnSGDL2P} and \subref{EnSGDLinfP} show the loss surface of EnSGD $-F(\bm{\theta})$ with the adversarial loss $\ell_{\varepsilon}(\bm{\theta})$
corresponding to Figs.~\ref{LinL2P} and \subref{LinLinfP}, respectively.
We computed the integral in EnSGD by using scipy~\cite{scipy}.
In contrast to Figs.~\ref{LinL2P} and \subref{LinLinfP}, loss functions for EnSGD are smooth.
Especially, although the smoothnesses of $\ell_{\varepsilon}(\bm{\theta})$ with $L_2$ and $L_\infty$ constraints are different,
smoothnesses of their loss functions $-F(\bm{\theta})$ are almost the same.
Thus, EnSGD smoothens non-smooth functions as shown in Theorem~\ref{EnThm}.
\subsection{Experimental setup for evaluation of EntropySGD}
This section gives an outline of the experimental conditions for evaluating EnSGD and 
the details are provided in the supplementary materials.
Our experimental codes are based on source codes provided by \citet{awp},
and our implementations of EnSGD are based on the codes provided by \citet{chaudhari2019entropy}.
Datasets of the experiments were CIFAR10, CIFAR100~\cite{cifar}, and SVHN~\cite{svhn}.
We compared the convergences of adversarial training when using SGD and EnSGD.
In addition, we evaluated the combination of EnSGD and AWP~\cite{awp},
which injects adversarial noise into the parameter to flatten the loss landscape.
We used ResNet-18 (RN18)~\cite{resnet} and WideResNet-34-10 (WRN)~\cite{WRN} following~\cite{awp}.
We used PGD, 
and the hyperparameters for PGD were based on~\cite{awp}.
The $L_\infty$ norm of the perturbation $\varepsilon\!=\!8/255$ at training time.
For EnSGD, we set $\gamma=0.03$, $\varepsilon_{E}=1\times10^{-4}$, and $\eta=0.1$, $\eta^\prime=0.1$ and
tuned an iteration $L$ in $\{20, 30\}$.
The learning rates of SGD and EnSGD are set to 0.1 and divided by 10 at the 100-th and 150-th epoch,
and we used early stopping by evaluating test robust accuracies against PGD (20-iteration).
The hyperparameter of AWP is tuned in $\{0.01, 0.005\}$.
For WRN, we used the same hyperparameters as those of RN18. 
We trained models three times and
show the average and standard deviation of test accuracies.
Note that we evaluated EnSGD for adversarial training since our analysis focuses on adversarial training
and also confirmed the effectiveness of EnSGD in TRADES~\cite{TRADES}, which is also used for adversarial robustness. 
The evaluation is provided in the supplementary materials.
\subsection{Results of EntropySGD}
Figure~\ref{TrainAcc} plots training and test accuracies on CIFAR10 attacked by PGD against epochs and runtime.
This figure shows that EnSGD accelerates the training and achieves higher accuracies than SGD.
This is because EnSGD improves the smoothness of adversarial loss as shown in Theorem~\ref{EnThm}.
In addition, EnSGD alleviates overfitting; i.e., 
training accuracy of AT+EnSGD is smaller than that of AT,
but test accuracy of AT+EnSGD is larger than that of AT.
When using EnSGD with AWP, EnSGD can alleviate underfitting.
In \rfig{TrainAccTime}, the runtime of EnSGD is almost the same as that of SGD.
In addition, AT with EnSGD is faster than AWP with SGD although they achieve similarly robust performance agasint AutoAttack in \rtab{AAtab}.
Table~\ref{AAtab} lists test robust accuracies against AutoAttack on CIFAR10 and against PGD on CIFAR100 and SVHN.
For almost all cases, EnSGD outperforms SGD: i.e., AT + EnSGD outperforms AT, and AWP + EnSGD outperforms AWP.
Though AWP outperforms AWP+EnSGD on CIFAR100, the difference is just 0.05.
\begin{figure}[tb]
    \centering
    \begin{subfigure}{0.63\linewidth}\centering
        \includegraphics[width=.94\linewidth]{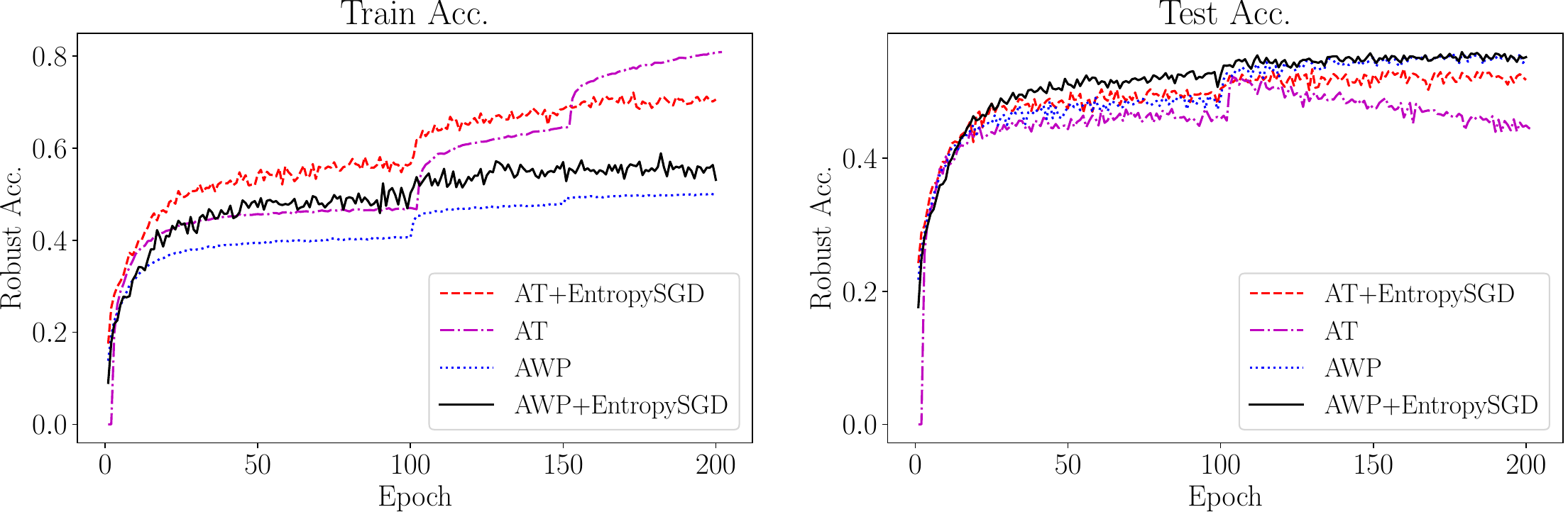}
        \caption{vs Epoch}
    \end{subfigure}
    \begin{subfigure}{0.33\linewidth}\centering
        \includegraphics[width=.94\linewidth]{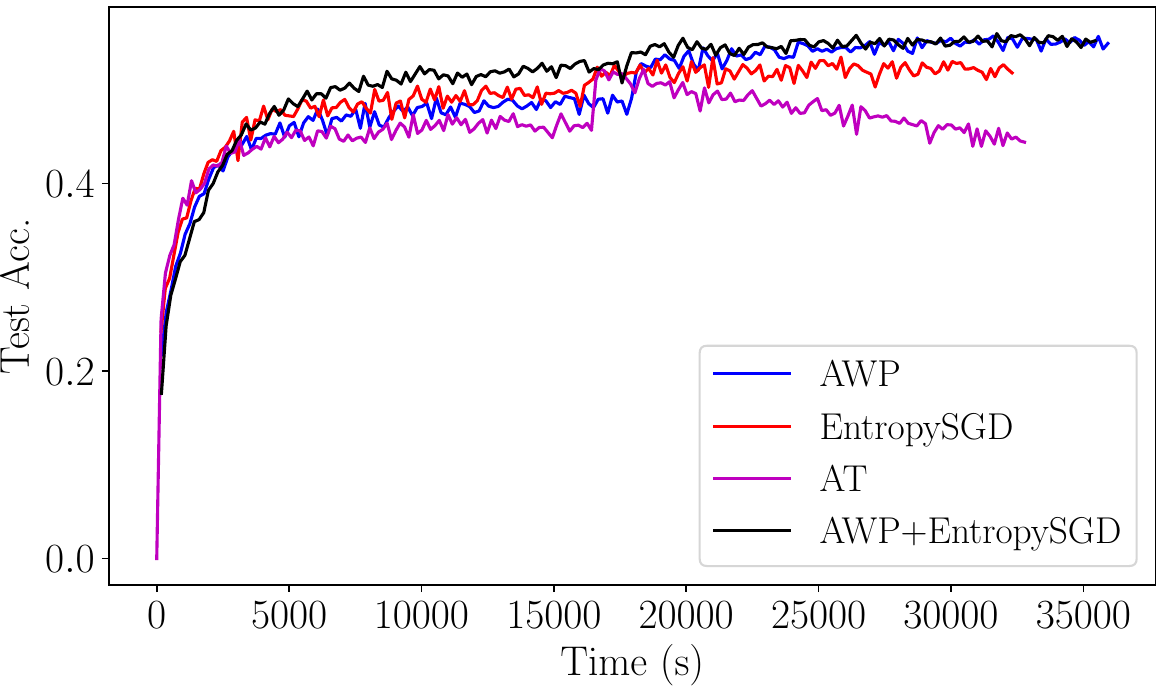}
        \caption{vs Runtime}    \label{TrainAccTime}
\end{subfigure}
    \caption{Robust accuracy against PGD vs (a) epochs and (b) runtime. 
    AT denotes adversarial training.
    Note that we use PGD with 10 iterations for training and PGD with 20 iterations for test.
    }
    \label{TrainAcc}
\end{figure}
\begin{table}[tbp]
    \centering
    \caption{Robust accuracies against AutoAttack (AA) on CIFAR10,
     and Robust accuracies against PGD on CIFAR100 and SVHN.}
    \label{AAtab}
    \begin{tabular}{cccccccc}\toprule
    &AT&AT+EnSGD&AWP&AWP+EnSGD\\\midrule
    CIFAR10 (AA, RN18)&$48.0\!\pm\! 0.2$&$49.2\!\pm\! 0.4$&$49.9 \!\pm\! 0.4$&$\bm{50.5\!\pm\! 0.2}$\\
    CIFAR10 (AA, WRN)&$51.9\!\pm\! 0.5$&$53.3\!\pm\! 0.3$&$53.2\!\pm\! 0.3$&$\bm{54.72\!\pm \!0.05}$\\\midrule
    SVHN (PGD, RN18)&$53.1\!\pm\! 0.6$&$59.1\!\pm\! 0.4$&$59.3\!\pm\! 0.1$&$\bm{59.91\!\pm\!0.09}$\\
    CIFAR100 (PGD, RN18)&$27.66\!\pm\! 0.04$&$28.69\!\pm\! 0.03$&$\bm{30.95\!\pm\! 0.07}$&$30.9 \!\pm\! 0.1$\\
    \bottomrule
    \end{tabular}
\end{table}
\section{Conclusion}
This paper investigated the smoothness of loss for adversarial training.
We proved that the smoothness of adversarial loss depends on the
constraints of adversarial examples, and the gradient of adversarial loss can be non-Lipschitz continuous
for some points.
Since smoothness of loss is an important property for gradient-based optimization, 
we showed that EntropySGD can improve the smoothness of adversarial loss.
Our usage of EntropySGD is still naive and not specialized for adversarial training.
In our future work, we will explore the smoothing methods for adversarial loss.
\bibliography{CameraBib}
\bibliographystyle{plainnat}
\appendix
\section{Proofs}
\setcounter{lemma}{0}
\begin{lemma}\label{AdvBase}
    If the adversarial example $\bm{x}^\prime (\bm{\theta})$ is a $C$-Lipschitz function, 
    the gradient of adversarial loss is $(C_{\bm{\theta}\bm{\theta}}+CC_{\bm{\theta}\bm{x}})$-Lipschitz, that is, adversarial loss is $(C_{\bm{\theta}\bm{\theta}}+CC_{\bm{\theta}\bm{x}})$-smooth.    
\end{lemma}
\begin{proof}
    If the adversarial example $\bm{x}^\prime (\bm{\theta})$ is a $C$-Lipschitz function, we have 
    \begin{align}\label{advC}
        ||\bm{x}^\prime(\bm{\theta}_1)-\bm{x}^\prime(\bm{\theta}_2)||\leq C||\bm{\theta}_1-\bm{\theta}_2||.
    \end{align}
    From Assumption~1, we have
    \begin{align}
        \!\!\!\textstyle||\nabla_{\bm{\theta}}\ell_{\varepsilon}(\bm{x},\bm{\theta}_1)&-\nabla_{\bm{\theta}}\ell_{\varepsilon}(\bm{x},\bm{\theta}_2)||=
        ||\nabla_{\bm{\theta}}\ell(\bm{x}^\prime(\bm{\theta}_1),\bm{\theta}_1)\!-\!\nabla_{\bm{\theta}}\ell(\bm{x}^\prime(\bm{\theta}_2),\bm{\theta}_2)||\nonumber\\
        &\leq 
        ||\nabla_{\bm{\theta}}\ell(\bm{x}^\prime(\bm{\theta}_1),\bm{\theta}_1)\!-\!\nabla_{\bm{\theta}}\ell(\bm{x}^\prime(\bm{\theta}_1),\bm{\theta}_2)||\!+\!||\nabla_{\bm{\theta}}\ell(\bm{x}^\prime(\bm{\theta}_1),\bm{\theta}_2)\!-\!\nabla_{\bm{\theta}}\ell(\bm{x}^\prime(\bm{\theta}_2),\bm{\theta}_2)||\nonumber
        \nonumber\\&\leq C_{\bm{\theta}\bm{\theta}}||\bm{\theta}_1-\bm{\theta}_2||+C_{\bm{\theta}\bm{x}}||\bm{x}^\prime(\bm{\theta}_1)-\bm{x}^\prime(\bm{\theta}_2)|| \nonumber
        \\&\leq C_{\bm{\theta}\bm{\theta}}||\bm{\theta}_1-\bm{\theta}_2||+C_{\bm{\theta}\bm{x}}C||\bm{\theta}_1-\bm{\theta}_2||
        \nonumber\\&\leq (C_{\bm{\theta}\bm{\theta}}+CC_{\bm{\theta}\bm{x}})||\bm{\theta}_1-\bm{\theta}_2||,
    \end{align}
    which completes the proof.
\end{proof}
\begin{lemma}
    \label{l2lem}
    Let $\bm{x}^\prime (\bm{\theta}_1)$ and $\bm{x}^\prime (\bm{\theta}_2)$ be adversarial examples around the data point $\bm{x}$ for $\bm{\theta}_1$ and $\bm{\theta}_2$, respectively.
    For adversarial training with $L_2$ norm constraints $||\bm{\delta}||_2\leq \varepsilon$, 
    if there exists a lower bound $\theta_{min}\in \mathbb{R}$ such as $||\bm{\theta}||_2\!\geq\! \theta_{\min}\!>\!0$, we have the following inequality:
    \begin{align}
     \textstyle   ||\bm{x}^\prime (\bm{\theta}_1)\!-\!\bm{x}^\prime (\bm{\theta}_2)||\!\leq\!\frac{\varepsilon }{\theta_{\min}}||\bm{\theta}_1\!-\!\bm{\theta}_2||.
    \end{align}
    Thus, adversarial examples are $ (\frac{\varepsilon }{\theta_{\min}})$-Lipschitz on a bounded set not including the origin $\bm{\theta}= \bm{0}$.
\end{lemma}
\begin{proof}
    First, we solve the following optimization problem to obtain the adversarial examples for the data point $(\bm{x},y)$:
    \begin{align}\textstyle
        \mathrm{arg}\!\max_{\bm{\delta}} &\textstyle\log \left( 1+\mathrm{exp}(-y \bm{\theta}^T (\bm{x} + \bm{\delta}))\right),\nonumber\\
        \nonumber &\textstyle \mathrm{subject~to~} ||\bm{\delta}||_2\leq \varepsilon.
    \end{align}
    We consider the case of $y=1$. Note that we can be easily derive the same results for $y=-1$.
By using the Lagrange multiplier, we use the following function:
    \begin{align}\textstyle
        J(\bm{\delta})=\mathrm{log}\left(1+\mathrm{exp}(-\bm{\theta}^T (\bm{x}+\bm{\delta}))\right)+\lambda(||\bm{\delta}||_2- \varepsilon),
    \label{L2Obj}
    \end{align}
    and we find the solution satisfying
    \begin{align}\textstyle
        \nabla_{\delta}J(\bm{\delta})=\bm{0}\label{KKT1},\\
        ||\bm{\delta}||_2- \varepsilon\leq 0\label{KKT2},\\
        \lambda\leq 0.\label{KKT3}
    \end{align}
From \req{KKT1}, we have
\begin{align}\textstyle
    \nabla_{\delta}J(\bm{\delta})=\frac{-\mathrm{exp}(-\bm{\theta}^T (\bm{x}+\bm{\delta}))}{1+\mathrm{exp}(-\bm{\theta}^T (\bm{x}+\bm{\delta}))}\bm{\theta}+\lambda \frac{\bm{\delta}}{||\bm{\delta}||_2}.
\end{align}
Since $\frac{-\mathrm{exp}(-\bm{\theta}^T (\bm{x}+\bm{\delta}))}{1+\mathrm{exp}(-\bm{\theta}^T (\bm{x}+\bm{\delta}))}$ and $\frac{\lambda}{||\bm{\theta}_2||}$ are scalar values and $\lambda\leq 0$, 
$\bm{\theta}$ and $\bm{\delta}$ have the opposite direction.
Thus, we can write as $\bm{\delta}=-k \bm{\theta}$ where $k\geq 0$.
Since $J$ monotonically increases in accordance with $k$, 
we have $k=\varepsilon$ from \req{KKT2}.
Therefore, we have $\bm{\delta}=-\varepsilon\frac{\bm{\theta}}{||\bm{\theta}||_2}$
Then, we have $\bm{x}^{\prime}(\bm{\theta})=\bm{x}-\varepsilon\frac{\bm{\theta}}{||\bm{\theta}||_2}$,
and thus, we compute the Lipschitz constants for $\bm{x}^{\prime}(\bm{\theta})=\bm{x}-\varepsilon\frac{\bm{\theta}}{||\bm{\theta}||_2}$.
Since Lipschitz constants for vector-valued continuous functions bound the operator norm of Jacobian,
we compute the Jacobian of $\bm{x}^{\prime}(\bm{\theta})=\bm{x}-\varepsilon\frac{\bm{\theta}}{||\bm{\theta}||_2}$.
We have 
\begin{align}\textstyle
    D_{\bm{\theta}}(\bm{x}^{\prime}(\bm{\theta}))=-\frac{\varepsilon}{||\bm{\theta}||_2}\left(\bm{I}-\frac{\bm{\theta}\bm{\theta}^T}{||\bm{\theta}||_2^2}\right).
\end{align}
The spectral norm of this matrix is $\frac{\varepsilon}{||\bm{\theta}||_2}$
because this matrix is a normal matrix that has eigenvalues of $\lambda_d (D_{\bm{\theta}}\bm{x}^{\prime}(\bm{\theta}))=0$ and
$\lambda_i(D_{\bm{\theta}}\bm{x}^{\prime}(\bm{\theta}))=-\frac{\varepsilon}{||\bm{\theta}||_2}$ for $i=1,\dots,d-1$.
Therefore, if we have $|| \bm{\theta}||_2\geq \theta_{\min}>0$, we have $\sup_{\bm{\theta}}\frac{\varepsilon}{||\bm{\theta}||_2}\leq\frac{\varepsilon}{\theta_{\min}}<\infty $.
Thus, we have
\begin{align}
    \textstyle   ||\bm{x}^\prime (\bm{\theta}_1)\!-\!\bm{x}^\prime (\bm{\theta}_2)||\!\leq\!\frac{\varepsilon }{\theta_{\min}}||\bm{\theta}_1\!-\!\bm{\theta}_2||.
   \end{align}
\end{proof}
\setcounter{theorem}{0}
\begin{theorem}\label{l2Thm}
    For adversarial training with $L_2$ norm constraints $||\bm{\delta}||_2\leq \varepsilon$,
    if we have $||\bm{\theta}||_2\geq \theta_{\min}>0$, the following inequality holds:
    \begin{align}
        \textstyle  \!\!||\nabla_{\bm{\theta}}L_{\varepsilon}(\bm{\theta}_1)\!-\!\nabla_{\bm{\theta}}L_{\varepsilon}(\bm{\theta}_2)||\!\leq\!(C_{\bm{\theta}\bm{\theta}}\!+\!\frac{\varepsilon C_{\bm{\theta}\bm{x}}}{\theta_{\min}})||\bm{\theta}_1\!-\!\bm{\theta}_2||.
    \end{align}
    Thus, adversarial loss $L_{\varepsilon}$ is $ (C_{\bm{\theta}\bm{\theta}}\!+\!\frac{\varepsilon C_{\bm{\theta}\bm{x}}}{\theta_{\min}})$-smooth on a bounded set that does not include $\bm{\theta}\!=\!\bm{0}$.
\end{theorem}
\begin{proof}
    We have
    $||\nabla_{\bm{\theta}}\ell_{\varepsilon}(\bm{x},\bm{\theta}_1)\!-\!\nabla_{\bm{\theta}}\ell_{\varepsilon}(\bm{x},\bm{\theta}_2)||\!\leq\!C_{\bm{\theta}\bm{\theta}}||\bm{\theta}_1\!-\!\bm{\theta}_2||\!+\!C_{\bm{\theta}\bm{x}}||\bm{x}^\prime(\bm{\theta}_1)\!-\!\bm{x}^\prime(\bm{\theta}_2)||$ from  $||\nabla_{\bm{\theta}}\ell_{\varepsilon}(\bm{x},\bm{\theta}_1)-\nabla_{\bm{\theta}}\ell_{\varepsilon}(\bm{x},\bm{\theta}_2)||\leq C_{\bm{\theta}\bm{\theta}}||\bm{\theta}_1-\bm{\theta}_2||+C_{\bm{\theta}\bm{x}}||\bm{x}^{1\prime}-\bm{x}^{2\prime}||$.
    From Lemmas~\ref{l2lem} and \ref{AdvBase}, we have 
    $||\bm{x}^\prime (\bm{\theta}_1)\!-\!\bm{x}^\prime (\bm{\theta}_2)||\!\leq\!\frac{\varepsilon }{\theta_{\min}}||\bm{\theta}_1\!-\!\bm{\theta}_2||$, and thus,
    \begin{align}\textstyle
        \!\!\!\!\!\!||\nabla_{\bm{\theta}}\ell_{\varepsilon}\textstyle(\bm{x},\bm{\theta}_1)\!-\!\nabla_{\bm{\theta}}\ell_{\varepsilon}(\bm{\bm{x},\theta}_2)||
        &\textstyle\leq \!C_{\bm{\theta}\bm{\theta}}||\bm{\theta}_1\!-\!\bm{\theta}_2||\!+\!C_{\bm{\theta}\bm{x}}\frac{\varepsilon }{\theta_{\min}}||\bm{\theta}_1\!-\!\bm{\theta}_2||
        \textstyle\!\leq\!(C_{\bm{\theta}\bm{\theta}}\!+\!\frac{\varepsilon C_{\bm{\theta}\bm{x}}}{\theta_{\min}})||\bm{\theta}_1\!-\!\bm{\theta}_2||.\nonumber
        \end{align}
Since $||\nabla_{\bm{\theta}} L_{\varepsilon}(\bm{\theta}_1)\!-\!\nabla_{\bm{\theta}} L_{\varepsilon}(\bm{\theta}_1)||\!=\!||\frac{1}{N}\sum_n (\nabla_{\bm{\theta}}\ell_{\varepsilon}(\bm{x}_n,\bm{\theta}_1)\!-\!
\nabla_{\bm{\theta}}\ell_{\varepsilon}(\bm{x}_n,\bm{\theta}_2))||\!=\!
\frac{1}{N}\sum_n ||\nabla_{\bm{\theta}}\ell_{\varepsilon}(\bm{x}_n,\bm{\theta}_n)\!-\!
\nabla_{\bm{\theta}}\ell_{\varepsilon}(\bm{x}_n,\bm{\theta}_2)||\leq (C_{\bm{\theta}\bm{\theta}}\!+\!\frac{\varepsilon C_{\bm{\theta}\bm{x}}}{\theta_{\min}})||\bm{\theta}_1\!-\!\bm{\theta}_2||$,
        which completes the proof.
\end{proof}
\begin{lemma}
    \label{linflem}
    The magnitude of adversarial perturbation is measured by the $L_\infty$ norm as $||\bm{\delta}||_\infty \leq \varepsilon$.
    If the sign of one element at least is different between $\bm{\theta}_1$ and $\bm{\theta}_2$ ($\exists i\!:\!\mathrm{sign}(\theta_{1,i})\neq \mathrm{sign}(\theta_{2,i})$),
    adversarial examples are not Lipschitz continuous.
    If all signs are the same as $\forall i\!:\!\mathrm{sign}(\theta_{1,i})= \mathrm{sign}(\theta_{2,i})$,
    we have the following equation:
    \begin{align}
        \textstyle  ||\bm{x}^\prime (\bm{\theta}_1)\!-\!\bm{x}^\prime (\bm{\theta}_2)||\!=\!0.
    \end{align}
    Thus, adversarial examples are Lipschitz continuous on a bounded set
    that does not include $\theta_i= 0,\forall i$ and where no signs of elements change.
\end{lemma}
\begin{proof}
    First, we solve the following optimization problem to obtain the adversarial examples for the data point $(\bm{x},y)$:
    \begin{align}\textstyle
        \mathrm{arg}\!\max_{\bm{\delta}} &\log \left( 1+\mathrm{exp}(-y \bm{\theta}^T (\bm{x} + \bm{\delta}))\right),\nonumber\\
        \nonumber &\mathrm{subject~to~} ||\bm{\delta}||_\infty\leq \varepsilon.
    \end{align}
    We consider the case of $y=1$. Note that we can easily derive the same results for $y=-1$.
    Since $\log(1+\mathrm{exp}(-\bm{\theta}^T\bm{x}-\bm{\theta}^T\bm{\delta}))$ is a monotonically decreasing
    function for $\bm{\theta}^T\bm{\delta}$,
    the solution minimizes $\bm{\theta}^T\bm{\delta}$ subject to $||\bm{\delta}||_{\infty}\leq \varepsilon$.
    The solution is obtained as $\bm{\delta}=-\varepsilon\mathrm{sign}(\bm{\theta})$ 
    where $\mathrm{sign}$ is an element-wise sign function \cite{fgsm}.
    Therefore, the optimal adversarial examples are 
    $\bm{x}^{\prime}=\bm{x}-\varepsilon\mathrm{sign}(\bm{\theta})$, and we investigate their Lipschitz continuity.
    Since $\mathrm{sign}(\theta_i)$ for $\theta_i>0$ or $\theta_i<0$ is a constant function,
    its derivative is zero for $\theta_i\neq 0$.
    At $\theta_i=0$, $\mathrm{sign}(\theta_i)$ is not continuous.
    Therefore, we have
    \begin{align}
        \textstyle  ||\bm{x}^\prime (\bm{\theta}_1)\!-\!\bm{x}^\prime (\bm{\theta}_2)||\!=\!0,
    \end{align}
    for the interval satisfying $\forall i\!:\!\mathrm{sign}(\theta_{1,i})= \mathrm{sign}(\theta_{2,i})$.
    If the signs of $\theta_i$ can change, adversarial examples are not Lipschitz continuous.
\end{proof}
\begin{theorem}\label{linfThm}
    For adversarial training with $L_\infty$ norm constraints $||\bm{\delta}||_\infty \leq \varepsilon$,
    in the interval where $\forall i:\mathrm{sign}(\theta_{1,i})= \mathrm{sign}(\theta_{2,i})$,
    the following inequality holds:
    \begin{align}
        \textstyle    ||\nabla_{\bm{\theta}}L_{\varepsilon}(\bm{\theta}_1)-\nabla_{\bm{\theta}}L_{\varepsilon}(\bm{\theta}_2)||&\leq C_{\bm{\theta}\bm{\theta}}||\bm{\theta}_1-\bm{\theta}_2||.
    \end{align}
    Thus, the loss function $L_{\varepsilon}$ for adversarial training is $ (C_{\bm{\theta}\bm{\theta}})$-smooth
    on a bounded set that does not include $\theta_i= 0,\forall i$ and where no signs of elements change.
\end{theorem}
\begin{proof}
    We have 
    $\!||\nabla_{\bm{\theta}}\ell_{\varepsilon}(\bm{\bm{x},\theta}_1)\!-\!\nabla_{\bm{\theta}}\ell_{\varepsilon}(\bm{x},\bm{\theta}_2)||\!\leq\!C_{\bm{\theta}\bm{\theta}}||\bm{\theta}_1\!-\!\bm{\theta}_2||\!+\!C_{\bm{\theta}\bm{x}}||\bm{x}^\prime\!(\bm{\theta}_1)\!-\!\bm{x}^\prime\!(\bm{\theta}_2)||$
     from $||\nabla_{\bm{\theta}}\ell_{\varepsilon}(\bm{x},\bm{\theta}_1)-\nabla_{\bm{\theta}}\ell_{\varepsilon}(\bm{x},\bm{\theta}_2)||\leq C_{\bm{\theta}\bm{\theta}}||\bm{\theta}_1-\bm{\theta}_2||+C_{\bm{\theta}\bm{x}}||\bm{x}^{1\prime}-\bm{x}^{2\prime}||$.
    From Lemmas~\ref{linflem} and \ref{AdvBase}, if all signs are the same as $\forall i\!:\!\mathrm{sign}(\theta_{1,i})\!=\!\mathrm{sign}(\theta_{2,i})$, we have 
    $||\bm{x}^\prime \!(\bm{\theta}_1)\!-\!\bm{x}^\prime\! (\bm{\theta}_2)||\!=\!0$, and thus,
        $||\nabla_{\bm{\theta}}\ell_{\varepsilon}(\bm{x},\bm{\theta}_1)\!-\!\nabla_{\bm{\theta}}\ell_{\varepsilon}(\bm{x},\bm{\theta}_2)||
        \leq \!C_{\bm{\theta}\bm{\theta}}||\bm{\theta}_1\!-\!\bm{\theta}_2||.$
        Otherwise, $||\bm{x}^\prime(\bm{\theta}_1)\!-\!\bm{x}^\prime(\bm{\theta}_2)||\leq 2\varepsilon$ because $||\bm{x}'(\bm{\theta})-\bm{x}||\leq \varepsilon$,
        and thus, the gradient of adversarial loss is not Lipschitz continuous, which completes the proof.
\end{proof}
\begin{lemma}
    We assume that $\nabla_x \ell(\bm{x},\bm{\theta})$ is a $C^1$ function and $\bm{x}^{\prime*}$ is the local maximum point satisfying $\nabla^2_x \ell (\bm{x}^{\prime*},\bm{\theta})\prec 0$\footnote{$\bm{A}\prec 0$ represents that $\bm{A}$ is negative definite.} inside the feasible region ($||\bm{x}^{\prime*}-\bm{x}||_p<\varepsilon$).
    If there is a constant $0<c<\infty$ such as $\max_i\lambda_i(\nabla^2_x \ell (\bm{x}^{\prime},\bm{\theta}))\leq -c$ where $\lambda_i$ is the $i$-th eigenvalue,
    the optimal adversarial example $\bm{x}^{\prime*}$ in some neighborhood $U$ of $\bm{\theta}$ is a continuously differentiable function $\bm{x}^\prime(\bm{\theta})$ and we have
    \begin{align}\textstyle
        ||\bm{x}^\prime(\bm{\theta}_1)-\bm{x}^\prime(\bm{\theta}_2)||\leq \frac{C_{\theta x}}{c}||\bm{\theta}_1-\bm{\theta}_2||~~\forall\bm{\theta}_1,\bm{\theta}_2\in U. 
    \end{align}
\end{lemma}
\begin{proof}
Since we assume that $\bm{x}^{\prime*}$ is the local maximum point inside the feasible regions, 
the constraints do not affect $\bm{x}^{\prime*}$. Thus, we have $\nabla_x \ell (\bm{x}^{\prime*} ,\bm{\theta})=\bm{0}$.
In addition, since we assume $\nabla^2_x \ell (\bm{x}^{\prime*},\bm{\theta})\prec 0$, we have $\mathrm{det}\nabla^2_x \ell (\bm{x}',\bm{\theta})\neq 0$.
From implicit function theorem, since $\mathrm{det}\nabla^2_x \ell (\bm{x}',\bm{\theta})\neq 0$,
there exists an open set $U$ containing $\bm{\theta}$ such that
there exists a unique continuously differentiable function $\bm{g}$ such that $\bm{x}^\prime=\bm{g}(\bm{\theta})$,
and $\nabla_x \ell (\bm{g}(\bm{\theta}) ,\bm{\theta})=\bm{0}$ for all $\bm{\theta}\in U$.
In addition, by using the implicit function theorem, its Jacobian is given by
\begin{align}\textstyle
    D_{\bm{\theta}}\bm{g}(\bm{\theta})=-\left(\nabla_x^2\ell(\bm{g}(\bm{\theta}),\bm{\theta}))\right)^{-1}\nabla_{\bm{\theta}} \nabla_{\bm{x}}\ell(\bm{g}(\bm{\theta}),\bm{\theta}).      
\end{align}
Since the upper bound of the operator norm of the Jacobian matrix becomes a Lipschitz constant,
we compute $\sigma_1(D_{\bm{\theta}}\bm{g}(\bm{\theta}))$.
From submultiplicativity, we have $\sigma_1(D_{\bm{\theta}}\bm{g}(\bm{\theta}))\leq \sigma_1(-\left(\nabla_x^2\ell(\bm{g}(\bm{\theta}),\bm{\theta})\right)^{-1})\sigma_1(\nabla_{\bm{\theta}} \nabla_{\bm{x}}\ell(\bm{g}(\bm{\theta}),\bm{\theta}))$.
Since the Hessian matrix is a normal matrix and we assume $\max_i\lambda_i(\nabla^2_x \ell (\bm{x}^{\prime},\bm{\theta}))\leq -c$, 
$\sigma_1(-\left(\nabla_x^2\ell(\bm{g}(\bm{\theta}),\bm{\theta})\right)^{-1})=\frac{1}{\min_i \sigma_i(\nabla_x^2\ell(\bm{g}(\bm{\theta}),\bm{\theta}))}=\frac{1}{\min_i |\lambda_i(\nabla_x^2\ell(\bm{g}(\bm{\theta}),\bm{\theta}))|}\leq \frac{1}{c}$.
On the other hand, we have $\sigma_1(\nabla_{\bm{\theta}} \nabla_{\bm{x}}\ell(\bm{g}(\bm{\theta}),\bm{\theta}))=\sigma_1(\nabla_{\bm{x}} \nabla_{\bm{\theta}}\ell(\bm{g}(\bm{\theta}),\bm{\theta}))$ since $\nabla_x \ell(\bm{x},\bm{\theta})$ is a $C^1$ function.
We have $\sigma_1(\nabla_{\bm{x}} \nabla_{\bm{\theta}}\ell(\bm{g}(\bm{\theta}),\bm{\theta}))\leq C_{\theta x}$ from Assumption~1 since we have
\begin{align}
    \sup_{\bm{\theta}} ||D_{\bm{\theta}} f(\bm{\theta})||&\textstyle\leq C_l,\label{JacoUp}
\end{align}
for a differentiable $C_l$-Lipschitz function $f$.
Therefore, the spectral norm of $D_{\bm{\theta}}\bm{g}(\bm{\theta})$ is bounded above by $C_{\theta x}/c$, and thus, we have
\begin{align}\textstyle
   ||\bm{x}^\prime(\bm{\theta}_1)-\bm{x}^\prime(\bm{\theta}_2)||\leq \sup_{\theta} \sigma_1(D_{\bm{\theta}}\bm{g}(\bm{\theta}))||\bm{\theta}_1-\bm{\theta}_2||\leq \frac{C_{\theta x}}{c}||\bm{\theta}_1-\bm{\theta}_2||,
\end{align}
which completes the proof.
\end{proof}
\begin{lemma}
    We assume that $\nabla_{x}\ell (\bm{x},\bm{\theta})$ is a $C^1$ function and $\tilde{\bm{x}}^{*}=[\mu^*,\bm{x}^{\prime*T}]^T$ is the local maximum point satisfying $\mathrm{det}\left(\nabla^2_{\tilde{x}} J (\bm{\tilde{x}}^{*},\bm{\theta})\right)>  0$
    on the boundary of the feasible regions of the $L_2$ constraint ($||\bm{x}^{\prime*}-\bm{x}||_2=\varepsilon$).
    If there is a constant $0<c<\infty$ such as $\min_i\sigma_i(\nabla^2_{\tilde{x}} J (\tilde{\bm{x}},\bm{\theta}))\geq c$ where $\sigma_i$ is the $i$-th singular value,
    the local maximum point $\bm{x}^{\prime*}$ in some neighborhood $U$ of $\bm{\theta}$ is a continuously differentiable function $\bm{x}^\prime(\bm{\theta})$ and we have
    \begin{align}
        ||\bm{x}^\prime(\bm{\theta}_1)-\bm{x}^\prime(\bm{\theta}_2)||\leq \frac{C_{\theta x}}{c}||\bm{\theta}_1-\bm{\theta}_2||~~\forall\bm{\theta}_1,\bm{\theta}_2\in U.
    \end{align}
    \end{lemma}
    \begin{proof}
    For adversarial loss with the $L_2$ constraint, if $\ell$ is twice differentiable, we can compute the bordered Hessian as
    \begin{align}\textstyle
        \nabla^2_{\tilde{x}} J (\bm{\tilde{x}},\bm{\theta})&\textstyle =
        \begin{bmatrix}
            0&\frac{(\bm{x}'-\bm{x})^T}{||\bm{x}'-\bm{x}||_2}\\
            \frac{\bm{x}'-\bm{x}}{||\bm{x}'-\bm{x}||_2}&\nabla_x^2\ell (\bm{x}^\prime,\bm{\theta})-\mu\frac{\bm{I}}{||\bm{x}'-\bm{x}||_2}+\mu\frac{(\bm{x}'-\bm{x})(\bm{x}'-\bm{x})^T}{||\bm{x}'-\bm{x}||_2^3}
        \end{bmatrix}.
    \end{align} 
    Since we assume that $\bm{x}^{\prime*}$ is the local maximum point on the boundary of the feasible regions, 
    $\tilde{\bm{x}}^{*}=[\mu^*,\bm{x}^{\prime*}]$ satisfies $\nabla_{\tilde{x}} J (\tilde{\bm{x}},\bm{\theta})=\bm{0}$.
    In addition, $\mathrm{det}\left(\nabla^2_{\tilde{x}} J (\bm{\tilde{x}}^*,\bm{\theta})\right)\neq 0$ since we assume $\mathrm{det}\left(\nabla^2_{\tilde{x}} J (\bm{\tilde{x}}^*,\bm{\theta})\right)>  0$.
    From implicit function theorem, since $\mathrm{det}\left(\nabla^2_{\tilde{x}} J (\bm{\tilde{x}}^*,\bm{\theta})\right)\neq 0$,
    there exists an open set $U$ containing $\bm{\theta}$ such that
    there exists a unique continuously differentiable function $\bm{\tilde{g}}$ such that $\tilde{\bm{x}}=\bm{\tilde{g}}(\bm{\theta})$,
    and $\nabla_{\tilde{x}} J(\tilde{\bm{g}}(\bm{\theta}) ,\bm{\theta})=\bm{0}$ for all $\bm{\theta}\in U$.
    In addition, by using the implicit function theorem, its Jacobian is given by
    \begin{align}\textstyle
       D_{\bm{\theta}}\tilde{\bm{g}}(\bm{\theta})=-\left(\nabla_{\tilde{x}}^2J (\tilde{\bm{g}}(\bm{\theta}),\bm{\theta}))\right)^{-1}\nabla_{\bm{\theta}} \nabla_{\tilde{\bm{x}}}J(\tilde{\bm{g}}(\bm{\theta}),\bm{\theta}).      
    \end{align}
    Since the upper bound of the operator norm of the Jacobian matrix becomes a Lipschitz constant,
    we compute $\sigma_1(D_{\bm{\theta}}\tilde{\bm{g}}(\bm{\theta}))$.
    From submultiplicativity, we have $\sigma_1(D_{\bm{\theta}}\tilde{\bm{g}}(\bm{\theta}))\leq \sigma_1(-\left(\nabla_x^2J (\tilde{\bm{g}}(\bm{\theta}),\bm{\theta})\right)^{-1})\sigma_1(\nabla_{\bm{\theta}} \nabla_{\tilde{\bm{x}}}J(\tilde{\bm{g}}(\bm{\theta}),\bm{\theta}))$.
    Since we assume $\min_i\sigma_i(\nabla^2_{\tilde{x}} J (\tilde{\bm{x}},\bm{\theta}))\geq c$, 
    $\sigma_1(-\left(\nabla_x^2J(\tilde{\bm{g}}(\bm{\theta}),\bm{\theta})\right)^{-1})=\frac{1}{\min_i\sigma_{i}(\nabla_{\tilde{x}}^2J(\tilde{\bm{g}}(\bm{\theta}),\bm{\theta}))}\leq \frac{1}{c}$.
    On the other hand, we have $\sigma_1(\nabla_{\bm{\theta}} \nabla_{\tilde{\bm{x}}} J(\tilde{\bm{g}}(\bm{\theta}),\bm{\theta}))=\sigma_1( \nabla_{\tilde{\bm{x}}}\nabla_{\bm{\theta}} J(\tilde{\bm{g}}(\bm{\theta}),\bm{\theta}))$  since $\nabla_{x}\ell (\bm{x},\bm{\theta})$ is a $C^1$ function.
    From Assumption~1 and \req{JacoUp}, we have $
    \sigma_1(\nabla_{\tilde{\bm{x}}} \nabla_{\bm{\theta}}J(\tilde{\bm{g}}(\bm{\theta}),\bm{\theta}))=\sigma_1([\bm{0},\nabla_x \nabla_{\bm{\theta}}\ell (\bm{x}^{\prime},\bm{\theta})]^T )\leq C_{\theta x}$.
    Therefore, the spectral norm of $D_{\bm{\theta}}\tilde{\bm{g}}(\bm{\theta})$ is bounded above by $C_{\theta x}/c$, and thus, we have
    \begin{align}\textstyle
        ||\bm{x}^\prime(\bm{\theta}_1)-\bm{x}^\prime(\bm{\theta}_2)||&\textstyle=\sqrt{\sum_i (x_i^\prime(\bm{\theta}_1)-x_i^\prime(\bm{\theta}_2))^2}\nonumber\\
        &\textstyle\leq\sqrt{\sum_i (x_i^\prime(\bm{\theta}_1)-x_i^\prime (\bm{\theta}_2))^2+(\mu(\bm{\theta}_1)-\mu(\bm{\theta}_2))^2 } \nonumber\\
        &\textstyle= ||\tilde{\bm{x}}(\bm{\theta}_1)-\tilde{\bm{x}}(\bm{\theta}_2)||\nonumber\\
        &\textstyle\leq \sup_{\theta} \sigma_1(D_{\bm{\theta}}\tilde{\bm{g}}(\bm{\theta}))||\bm{\theta}_1-\bm{\theta}_2||\leq \frac{C_{\theta x}}{c}||\bm{\theta}_1-\bm{\theta}_2||,
    \end{align}
    which completes the proof.
    \end{proof}
    \begin{theorem}\label{EnThm}
        Let $\bm{\Sigma}_{\bm{\theta}^\prime}$ be a variance-covariance matrix of $p_{\bm{\theta}}(\bm{\theta}^\prime )$ in \req{PropEnSGD}.
        If we use EntropySGD for the non-negative loss function $L(\bm{\theta})\!\geq\! 0$, we have 
        \begin{align}\textstyle
            \!\!\!||\nabla_{\bm{\theta}} F(\bm{\theta}_1)\!-\!\nabla_{\bm{\theta}} F(\bm{\theta}_2))||\!\leq \!\left(\gamma \!+\! \gamma^2 \sup_{\bm{\theta}} ||\bm{\Sigma_{\bm{\theta}^\prime}}||_F \right )||\bm{\theta}_1\!-\!\bm{\theta}_2||, 
        \end{align} 
        and $||\bm{\Sigma}_{\bm{\theta}^\prime}||_{F}\!< \!\infty$ for $\bm{\theta}\!\in\! \mathbb{R}^m$.
        Thus, $-\nabla_{\bm{\theta}} F(\bm{\theta})$ is Lipschitz continuous on a bounded set of $\bm{\theta}$.
    \end{theorem}
    \begin{proof}
From \req{HesUp}, the loss of EntropySGD is smooth if the spectral norm of the Hessian matrix of loss are bounded above by finite values.
Thus, we investigate the Hessian matrix of EntropySGD.
The gradient of EntropySGD is
\begin{align}\textstyle
    -\nabla_{\bm{\theta}} F(\bm{\theta},\gamma)=\gamma(\bm{\theta}-\mathbb{E}_{p_{\bm{\theta}}(\bm{\theta}')}[\bm{\theta}']),\label{GradEnSGD}
\end{align}
where $p_{\bm{\theta}}$ is a probability density function as
\begin{align}\textstyle
    p_{\bm{\theta}}(\bm{\theta}')=\frac{1}{e^{F(\bm{\theta},\gamma)}}\mathrm{exp} \left(- L(\bm{\theta}')-\frac{\gamma}{2}||\bm{\theta}\!-\!\bm{\theta}'||_2^2 \right).
   \label{PropEnSGD}
\end{align}
Note that we use the following equality for the derivation of the gradient:
\begin{align}
    \label{Average}
    \mathbb{E}_{p_{\bm{\theta}}}[f(\bm{\theta}^{\prime})]\!=\!\frac{1}{e^F}\!\int\!\!f(\bm{\theta}^\prime) \!\exp(-L(\bm{\theta}^{\prime})\!-\!\frac{\gamma}{2}||\bm{\theta}\!-\!\bm{\theta}^{\prime}||^2_2)d\bm{\theta}^{\prime}.
\end{align}
To compute the Hessian $-\nabla^2_{\bm{\theta}} \! F$, we compute the derivative of \req{GradEnSGD} as
\begin{align}
    -\frac{\partial^2\!F}{\partial \theta_j\partial \theta_i }&\!=\! \frac{\partial}{\partial \theta_j}\gamma (\theta_{i} -\mathbb{E}_{p_{\bm{\theta}}}[\theta^{\prime}_i])zw
    \!=\!\gamma \delta_{ij}\!-\!\gamma\frac{\partial}{\partial \theta_j}\mathbb{E}_{p_{\bm{\theta}}}[\theta^{\prime}_i]\\
    &\!\!\!\!\!\!=\!\gamma \delta_{ij}\!-\!\gamma\frac{\partial}{\partial \theta_j}\!\frac{1}{e^F}\!\!\!\int\!\!\theta^{\prime}_i\!\exp(\!-\!L(\bm{\theta}^{\prime})\!-\!\frac{\gamma}{2}||\bm{\theta}\!-\!\bm{\theta}^{\prime}||^2_2)d\bm{\theta}^{\prime}.
\end{align}
By using \req{Average}, the second term becomes 
\begin{align}
    &\frac{\partial}{\partial \theta_j} \frac{1}{e^F} \int \theta^{\prime}_i\exp(-L(\bm{\theta}^{\prime})-\frac{\gamma}{2}||\bm{\theta}-\bm{\theta}^{\prime}||^2_2)d\bm{\theta}^{\prime}\nonumber\\
    = &-\frac{1}{e^{2F}} (\frac{\partial F}{\partial \theta_j} )e^F \int \theta^{\prime}_i\exp(-L(\bm{\theta}^{\prime})-\frac{\gamma}{2}||\bm{\theta}-\bm{\theta}^{\prime}||^2_2)d\bm{\theta}^{\prime}\nonumber\\
&+   \frac{1}{e^F} \int \frac{\partial}{\partial \theta_j}  \theta^{\prime}_i\exp(-L(\bm{\theta}^{\prime})-\frac{\gamma}{2}||\bm{\theta}-\bm{\theta}^{\prime}||^2_2)d\bm{\theta}^{\prime}\nonumber\\
=& \gamma (\theta_j \!-\!\mathbb{E}_{p_{\bm{\theta}}}[\theta^{\prime}_j] )\frac{1}{e^{F}}\!\!\int\!\theta^{\prime}_i\exp(-L(\bm{\theta}^{\prime})\!-\!\frac{\gamma}{2}||\bm{\theta}-\bm{\theta}^{\prime}||^2_2)d\bm{\theta}^{\prime}\nonumber\\
&+\frac{1}{e^F}\!\!\int\!\!-\gamma(\theta_j-\theta_j^{\prime}) \theta^{\prime}_i\exp(-L(\bm{\theta}^{\prime})\!-\!\frac{\gamma}{2}||\bm{\theta}-\bm{\theta}^{\prime}||^2_2)d\bm{\theta}^{\prime}\nonumber\\
=& \gamma (\theta_j -\mathbb{E}_{p_{\bm{\theta}}}[\theta^{\prime}_j] )\mathbb{E}_{p_{\bm{\theta}}}[\theta^{\prime}_i]
-\gamma\mathbb{E}_{p_{\bm{\theta}}}[(\theta_j-\theta_j^{\prime}) \theta^{\prime}_i]\nonumber\\
=& \gamma\theta_j\mathbb{E}_{p_{\bm{\theta}}}[\theta^{\prime}_i]\!-\!\gamma \mathbb{E}_{p_{\bm{\theta}}}[\theta^{\prime}_j] \mathbb{E}_{p_{\bm{\theta}}}[\theta^{\prime}_i]
\!-\!\gamma\theta_j\mathbb{E}_{p_{\bm{\theta}}}[\theta^{\prime}_i]\!+\!\gamma\mathbb{E}_{p_{\bm{\theta}}}[\theta_j^{\prime}\theta^{\prime}_i ]\nonumber\\
=& -\gamma \mathbb{E}_{p_{\bm{\theta}}}[\theta^{\prime}_j] \mathbb{E}_{p_{\bm{\theta}}}[\theta^{\prime}_i]+\gamma\mathbb{E}_{p_{\bm{\theta}}}[\theta_j^{\prime}\theta^{\prime}_i ]\nonumber\\
=& \gamma (\mathbb{E}_{p_{\bm{\theta}}}[\theta_j^{\prime}\theta^{\prime}_i ]-\mathbb{E}_{p_{\bm{\theta}}}[\theta^{\prime}_j] \mathbb{E}_{p_{\bm{\theta}}}[\theta^{\prime}_i]).
\end{align}
Therefore, we have
\begin{align}
    -\frac{\partial^2 F}{\partial \theta_j \partial \theta_i }&=\gamma \delta_{ij} -\gamma^2 (\mathbb{E}_{p_{\bm{\theta}}}[\theta_j^{\prime}\theta^{\prime}_i ]-\mathbb{E}_{p_{\bm{\theta}}}[\theta^{\prime}_j] \mathbb{E}_{p_{\bm{\theta}}}[\theta^{\prime}_i]).\nonumber
\end{align}
Since $\mathbb{E}_{p_{\bm{\theta}}}[\bm{\theta}'\bm{\theta}'^T]-\mathbb{E}_{p_{\bm{\theta}}}[\bm{\theta}']\mathbb{E}_{p_{\bm{\theta}}}[\bm{\theta}'^T]$
is a variance-covariance matrix of $p_{\bm{\theta}}(\bm{\theta}^{\prime})$, we have
\begin{align}
    -\nabla^2_{\bm{\theta}} F=\gamma \bm{I}-\gamma^2 \bm{\Sigma}_{\bm{\theta}}.
\end{align}
Thus, for the spectral norm of $-\nabla^2_{\bm{\theta}} F$, we have 
\begin{align}
    \sigma_1(-\nabla^2_{\bm{\theta}} F)= \sigma_1(\gamma \bm{I}-\gamma^2 \bm{\Sigma}_{\bm{\theta}})\leq \gamma \sigma_1(\bm{I})+\sigma_1(-\gamma^2 \bm{\Sigma}_{\bm{\theta}})=\gamma +\gamma^2 \sigma_1(\bm{\Sigma}_{\bm{\theta}}).
\end{align}
Since spectral norms of matrices are smaller than or equal to Frobenius norms, we have
\begin{align}
    \gamma +\gamma^2\sigma_1( \bm{\Sigma}_{\bm{\theta}}) \leq \gamma+\gamma^2||\bm{\Sigma}_{\bm{\theta}}||_F.
\end{align}
Thus, if all elements of $\bm{\Sigma}_{\bm{\theta}}$ are finite values, $\gamma+\gamma^2||\bm{\Sigma}_{\bm{\theta}}||_F$
becomes a Lipschitz constant.
Therefore, we will show 
\begin{align}\label{bound}
    -\infty<\left[\bm{\Sigma}_{\bm{\theta}}\right]_{i,j} =\mathbb{E}_p[\theta^{\prime}_i\theta^{\prime}_j]-\mathbb{E}_p[\theta^{\prime}_i]\mathbb{E}_p[\theta^{\prime}_j]<\infty.
\end{align}
If we have $|\mathbb{E}_p[\theta^{\prime}_i\theta^{\prime}_j]|<\infty$ and $|\mathbb{E}_p[\theta^{\prime}_i]\mathbb{E}_p[\theta^{\prime}_j]|<\infty$,
\req{bound} is satisfied.
From \req{Average}, $|\mathbb{E}_p[\theta^{\prime}_i\theta^{\prime}_j]|<\infty$ and $|\mathbb{E}_p[\theta^{\prime}_i]\mathbb{E}_p[\theta^{\prime}_j]|<\infty$ are satisfied
if the following improper integrals converge:
\begin{align}
    \int  \theta^{\prime}_i\theta^{\prime}_j\mathrm{exp}(-L(\bm{\theta^{\prime}})-\frac{\gamma}{2}||\bm{\theta}-\bm{\theta^{\prime}}||_2^2)d\bm{\theta^{\prime}},\label{Integ1}\\
    \int  \theta^{\prime}_i\mathrm{exp}(-L(\bm{\theta^{\prime}})-\frac{\gamma}{2}||\bm{\theta}-\bm{\theta^{\prime}}||_2^2)d\bm{\theta^{\prime}}.\label{Integ2}
\end{align}
Note that we ignore the division by $e^F$ because $e^F$ is bounded as $0<e^F<\infty$.
First, we consider \req{Integ1}.
Since $\theta^{\prime}_i\theta^{\prime}_j$
 becomes negative and is difficult to compute in the improper integral, 
we split the integral into two integrals: 
\begin{align}
    \!\!\int  \theta^{\prime}_i\theta^{\prime}_j\mathrm{exp}(-L(\bm{\theta^{\prime}})-\frac{\gamma}{2}||\bm{\theta}-\bm{\theta^{\prime}}||_2^2)d\bm{\theta^{\prime}}
    &=\int_{\bm{\Theta}_+}   \theta^{\prime}_i\theta^{\prime}_j\mathrm{exp}(-L(\bm{\theta^{\prime}})-\frac{\gamma}{2}||\bm{\theta}-\bm{\theta^{\prime}}||_2^2)d\bm{\theta^{\prime}}\label{Posi1}\\
    &-\int_{\bm{\Theta}_{-}} -\theta^{\prime}_i\theta^{\prime}_j\mathrm{exp}(-L(\bm{\theta^{\prime}})-\frac{\gamma}{2}||\bm{\theta}-\bm{\theta^{\prime}}||_2^2)d\bm{\theta^{\prime}}\label{Nega1},
\end{align}
where $\bm{\Theta}_{+}$ is the positive part: $\bm{\Theta}_{+}=\{\bm{\theta}| 0\!\leq\!\theta^\prime_i\!<\!\infty, 0\!\leq\!\theta^\prime_j\!<\!\infty, -\infty\!<\!\theta^\prime_k\!<\!\infty~\mathrm{for}~k\neq i,j\}\cup\{\bm{\theta}| -\infty\!<\!\theta^\prime_i\!<\!0, -\infty\!<\!\theta^\prime_j\!<\!0, -\infty\!<\!\theta^\prime_k\!<\!\infty~\mathrm{for}~k\neq i,j\}$
and $\bm{\Theta}_{-}$ is the negative part:  $\bm{\Theta}_{-}=\{\bm{\theta}| 0\!\leq\!\theta^\prime_i\!<\!\infty, -\infty\!<\!\theta^\prime_j\!<\!0, -\infty\!<\!\theta^\prime_k\!<\!\infty~\mathrm{for}~k\neq i,j\}\cup\{\bm{\theta}| -\infty\!<\!\theta^\prime_i\!<\!0, 0\!\leq\!\theta^\prime_j\!<\!\infty, -\infty\!<\!\theta^\prime_k\!<\!\infty~\mathrm{for}~k\neq i,j\}$.
For the positive part \req{Posi1}, we have 
\begin{align}
    \int_{\bm{\Theta}_+}   \theta^{\prime}_i\theta^{\prime}_j\mathrm{exp}(-L(\bm{\theta^{\prime}})-\frac{\gamma}{2}||\bm{\theta}-\bm{\theta^{\prime}}||_2^2)d\bm{\theta^{\prime}}
    \leq \int_{\bm{\Theta}_+}   \theta^{\prime}_i\theta^{\prime}_j\mathrm{exp}(-\frac{\gamma}{2}||\bm{\theta}-\bm{\theta^{\prime}}||_2^2)d\bm{\theta^{\prime}},\label{UpBoundPosi1}
\end{align}
because $0\leq L(\bm{\theta^{\prime}})$, $0\leq \mathrm{exp}(-L(\bm{\theta^{\prime}}))\leq 1$, and $\theta_i\theta_j\geq 0$ for $\bm{\theta}\in \bm{\Theta}_{+}$.
Since the right hand side of \req{UpBoundPosi1} is a part of the computation of the covariance (variance) of Gaussian distributions $\mathcal{N}(\bm{\theta},\gamma\bm{I})$,
we have 
\begin{align}
    \int_{\bm{\Theta}_+}   \theta^{\prime}_i\theta^{\prime}_j\mathrm{exp}(-L(\bm{\theta^{\prime}})-\frac{\gamma}{2}||\bm{\theta}-\bm{\theta^{\prime}}||_2^2)d\bm{\theta^{\prime}}
    \leq \int_{\bm{\Theta}_+}   \theta^{\prime}_i\theta^{\prime}_j\mathrm{exp}(-\frac{\gamma}{2}||\bm{\theta}-\bm{\theta^{\prime}}||_2^2)d\bm{\theta^{\prime}}<\infty,
\end{align}
for $\forall \bm{\theta}\in \mathbb{R}^m$.
Similarly, for negative part \req{Nega1}, we have 
\begin{align}
    \int_{\bm{\Theta}_-}   -\theta^{\prime}_i\theta^{\prime}_j\mathrm{exp}(-L(\bm{\theta^{\prime}})-\frac{\gamma}{2}||\bm{\theta}-\bm{\theta^{\prime}}||_2^2)d\bm{\theta^{\prime}}
    \leq \int_{\bm{\Theta}_-}   -\theta^{\prime}_i\theta^{\prime}_j\mathrm{exp}(-\frac{\gamma}{2}||\bm{\theta}-\bm{\theta^{\prime}}||_2^2)d\bm{\theta^{\prime}}<\infty,\label{UpBoundNega1}
\end{align}
because $\theta_i\theta_j\leq 0$ for $\bm{\theta}\in \bm{\Theta}_{-}$. Therefore,
we have 
\begin{align}
    -\infty<\int  \theta^{\prime}_i\theta^{\prime}_j\mathrm{exp}(-L(\bm{\theta^{\prime}})-\frac{\gamma}{2}||\bm{\theta}-\bm{\theta^{\prime}}||_2^2)d\bm{\theta^{\prime}}< \infty,
\end{align}
and thus, $\mathbb{E}_p[\theta^{\prime}_i\theta^{\prime}_j]$ converges to a finite value.
Next, we consider \req{Integ2}.
Similarly, we split the integral into two integrals: 
\begin{align}
    \!\!\int  \theta^{\prime}_i\mathrm{exp}(-L(\bm{\theta^{\prime}})-\frac{\gamma}{2}||\bm{\theta}-\bm{\theta^{\prime}}||_2^2)d\bm{\theta^{\prime}}
    &=\int_{\bm{\Theta}_+}   \theta^{\prime}_i\mathrm{exp}(-L(\bm{\theta^{\prime}})-\frac{\gamma}{2}||\bm{\theta}-\bm{\theta^{\prime}}||_2^2)d\bm{\theta^{\prime}}\label{Posi2}\\
    &-\int_{\bm{\Theta}_{-}} -\theta^{\prime}_i\mathrm{exp}(-L(\bm{\theta^{\prime}})-\frac{\gamma}{2}||\bm{\theta}-\bm{\theta^{\prime}}||_2^2)d\bm{\theta^{\prime}}\label{Nega2},
\end{align}
where $\bm{\Theta}_{+}$ is the positive part: $\bm{\Theta}_{+}=\{\bm{\theta}| 0\!\leq\!\theta^\prime_i\!<\!\infty, -\infty\!<\!\theta^\prime_k\!<\!\infty~\mathrm{for}~k\neq i\}$
and $\bm{\Theta}_{-}$ is the negative part:  $\bm{\Theta}_{-}=\{\bm{\theta}| -\infty\!\leq\!\theta^\prime_i\!<\!0, -\infty\!<\!\theta^\prime_k\!<\!\infty~\mathrm{for}~k\neq i\}$.
In the same manner, we have 
\begin{align}
    \int_{\bm{\Theta}_+}   \theta^{\prime}_i\mathrm{exp}(-L(\bm{\theta^{\prime}})-\frac{\gamma}{2}||\bm{\theta}-\bm{\theta^{\prime}}||_2^2)d\bm{\theta^{\prime}}
    &\leq \int_{\bm{\Theta}_+}   \theta^{\prime}_i\mathrm{exp}(-\frac{\gamma}{2}||\bm{\theta}-\bm{\theta^{\prime}}||_2^2)d\bm{\theta^{\prime}}<\infty\label{UpBoundPosi2}\\
    \int_{\bm{\Theta}_-}   -\theta^{\prime}_i\mathrm{exp}(-L(\bm{\theta^{\prime}})-\frac{\gamma}{2}||\bm{\theta}-\bm{\theta^{\prime}}||_2^2)d\bm{\theta^{\prime}}
    &\leq \int_{\bm{\Theta}_-}   -\theta^{\prime}_i\mathrm{exp}(-\frac{\gamma}{2}||\bm{\theta}-\bm{\theta^{\prime}}||_2^2)d\bm{\theta^{\prime}}<\infty\label{UpBoundNega2}
\end{align}
because $\int_{\bm{\Theta}_*}   \theta^{\prime}_i\mathrm{exp}(-\frac{\gamma}{2}||\bm{\theta}-\bm{\theta^{\prime}}||_2^2)d\bm{\theta^{\prime}}$ 
is a part of the computation of the mean of Gaussian distributions $\mathcal{N}(\bm{\theta},\gamma\bm{I})$.
From Eqs.~(\ref{UpBoundPosi2}) and (\ref{UpBoundNega2}), $\mathbb{E}_p[\theta^{\prime}_i]$ also converges to a finite value.
From the above, we have
\begin{align}
    |\mathbb{E}_p[\theta^{\prime}_i\theta^{\prime}_j]|<\infty,\\
    |\mathbb{E}_p[\theta^{\prime}_i]\mathbb{E}_p[\theta^{\prime}_j]|<\infty,
\end{align}
 and thus, for a bounded set of $\bm{\theta }$, we have
\begin{align}
    \sigma_1(\bm{\Sigma}_{\bm{\theta}})\leq ||\bm{\Sigma}_{\bm{\theta}}||_F<\infty.
\end{align}
Finally, we have 
\begin{align}
    \sigma_1(-\nabla^2_{\bm{\theta}} F)\leq \gamma+\gamma^2||\bm{\Sigma}_{\bm{\theta}}||_F<\infty,
\end{align}
which completes the proof.
\end{proof}
    \section{Relation between smoothness and flatness}
    Lipschitz continuity of the gradient of an objective function 
    is required for effective optimization by gradient-based methods \cite{NLProgram}.
    On the other hand, the flatness of loss landscape is related to generalization performance of deep learning \cite{keskar2016large,awp,SHAM}.
    Sharpness of loss function is related to the spectral norm of Hessian matrices:
    \citet{dinh2017sharp} and \cite{zhang2021flatness} defined the sharpness as follows:
    \setcounter{definition}{1}
    \begin{definition}
        Let $B_2(\varepsilon,\bm{\theta})$ be an $L_2$ ball centered on a minimum $\bm{\theta}$ with radius $\varepsilon$.
        Then, for a non-negative valued loss function $L$, the $\varepsilon$-sharpness will be defined as proportional to
        \begin{align}
            \textstyle \frac{\max_{\bm{\theta}^\prime\in B_2(\varepsilon,\bm{\theta})}(L(\bm{\theta}^\prime)-L(\bm{\theta}))}{1+L(\bm{\theta})}.\label{sharp}
        \end{align}
        For small $\varepsilon$, \req{sharp} can be approximated by using the spectral norm of the Hessian matrix:
        \begin{align}
            \textstyle \frac{||\nabla_{\bm{\theta}}^2L(\bm{\theta})||_2\varepsilon^2}{2(1+L(\bm{\theta}))}.        
        \end{align}
    \end{definition}
    \citet{chaudhari2019entropy} also used the spectral norm of the Hessian matrix as a measure of flatness.
    The $C_s$-smooth function $L(\bm{\theta})$ with small $C_s$ tends to have smaller spectral norms of Hessian matrices 
    because the largest spectral norm of Hessian matrices is bounded by a Lipschitz constant of the gradient 
    if the gradient is an everywhere differentiable function
    \begin{align}\textstyle\label{HesUp}
        \textstyle\sup_{\bm{\theta}} \sigma_1 (\nabla^2_{\bm{\theta}} L(\bm{\theta}))&\textstyle\leq C_s.
    \end{align}
    Therefore, the smooth function tends to have flat loss landscapes.
    
    \citet{awp} have investigated the relationship between loss landscape and generalization performance of adversarial training.
    They have shown that adversarial training can sharpen the loss landscape, which degrades generalization performance.
    To improve the flatness of loss landscape, they have proposed adversarial weight perturbation (AWP).
    AWP trains models by minimizing the following objective function:
    \begin{align}
        \textstyle
    \!\!\!\!\max _{\bm{v} \in \mathcal{V}} L(\bm{\theta}\!+\!\mathbf{v})\!=\!\max _{\bm{v} \in \mathcal{V}}\! \frac{1}{N}\! \sum_i \ell_{\varepsilon}(\bm{x}_i,y_i,\bm{\theta}\!+\!\bm{v}),
    \end{align}
    where $\mathcal{V}$ is a feasible region for $\bm{v}$, and $\mathcal{V}_l\!=\!\{\bm{v}_l|~||\bm{v}_l||=\gamma_A||\bm{\theta}_l||\}$ for the $l$-th layer parameter $\bm{\theta}_l$.
    AWP can flatten the loss landscape and achieve good generative performance for adversarial training.
    However, since $\bm{v}$ is iteratively updated like PGD, AWP requires about 8\% training time overhead.
    \section{Additional toy examples for visualization of loss surfaces.}
 \rfig{Plot2} shows the loss surface for the dataset $L(\bm{\theta})=1/N \sum_n \ell(\bm{x}_n,y_n,\bm{\theta}) $ in the standard training, adversarial training with $L_2$ constraint, and
adversarial training with $L_\infty$ constraint when $d$ is set to two. 
This figure also shows that adversarial loss with the $L_2$ constraint is not continuous at $\bm{\theta}=\bm{0}$
and adversarial loss with the $L_\infty$ constraint is not continuous at $\theta_i=0$; i.e., the point where the sign of $\theta_i$ changes.
Therefore, the result follows Theorems~\ref{l2Thm} and \ref{linfThm}.
In \rfig{Plot2}, we set the label as $y_n=\mathrm{sign}(x_{n,1})$, and thus, the optimal $\theta_2^*$ for adversarial loss is 0
since the second feature of adversarial examples $x_{n,2}+\delta_{i,2}$ only degrades the classification performance.
In this case, adversarial training with $L_\infty$ constraints can suffer from the non-smoothness of the loss function
because the Lipschitz condition does not hold.
    \begin{figure*}[tb]
        \centering
        \subfloat[][Clean loss ]{\includegraphics[width=0.325\linewidth]{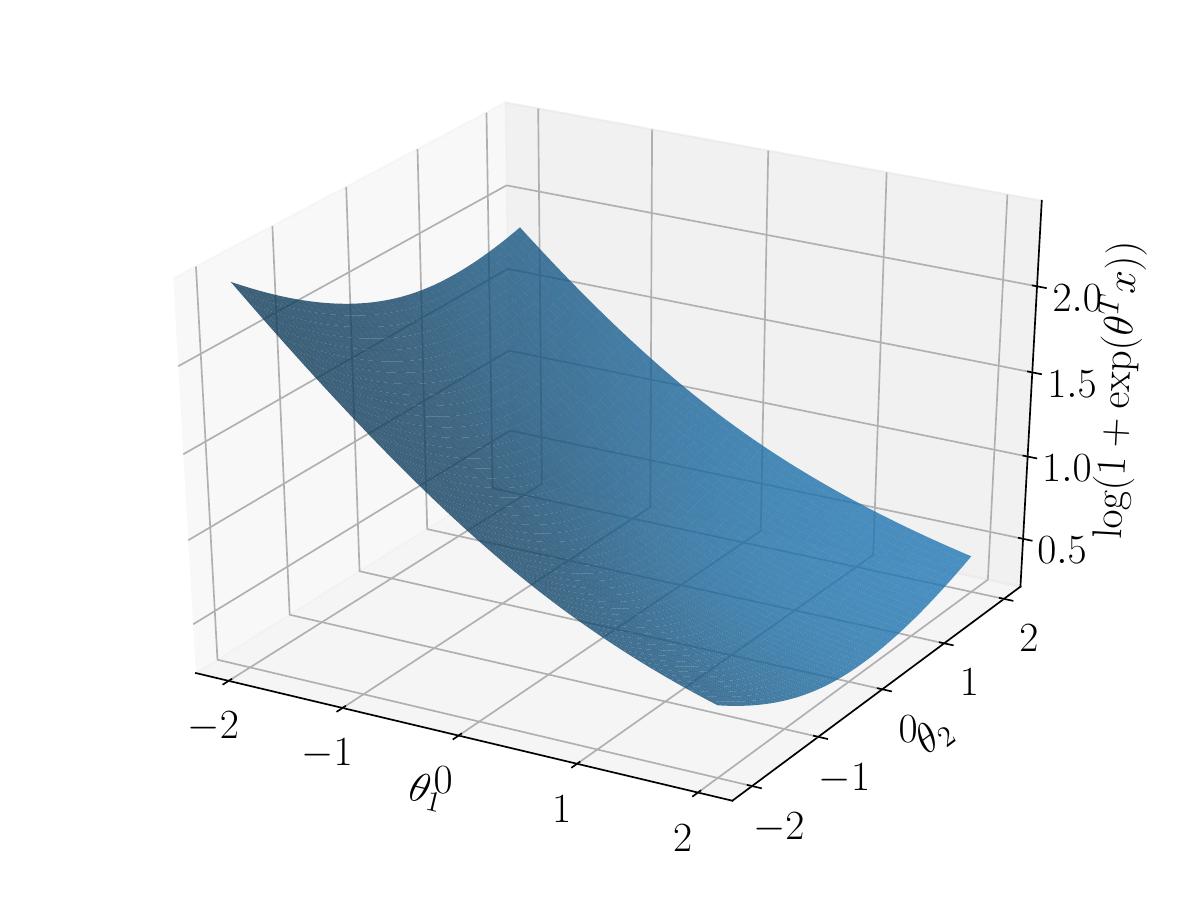}}
        \subfloat[][Adversarial loss with $L_2$]{\includegraphics[width=0.325\linewidth]{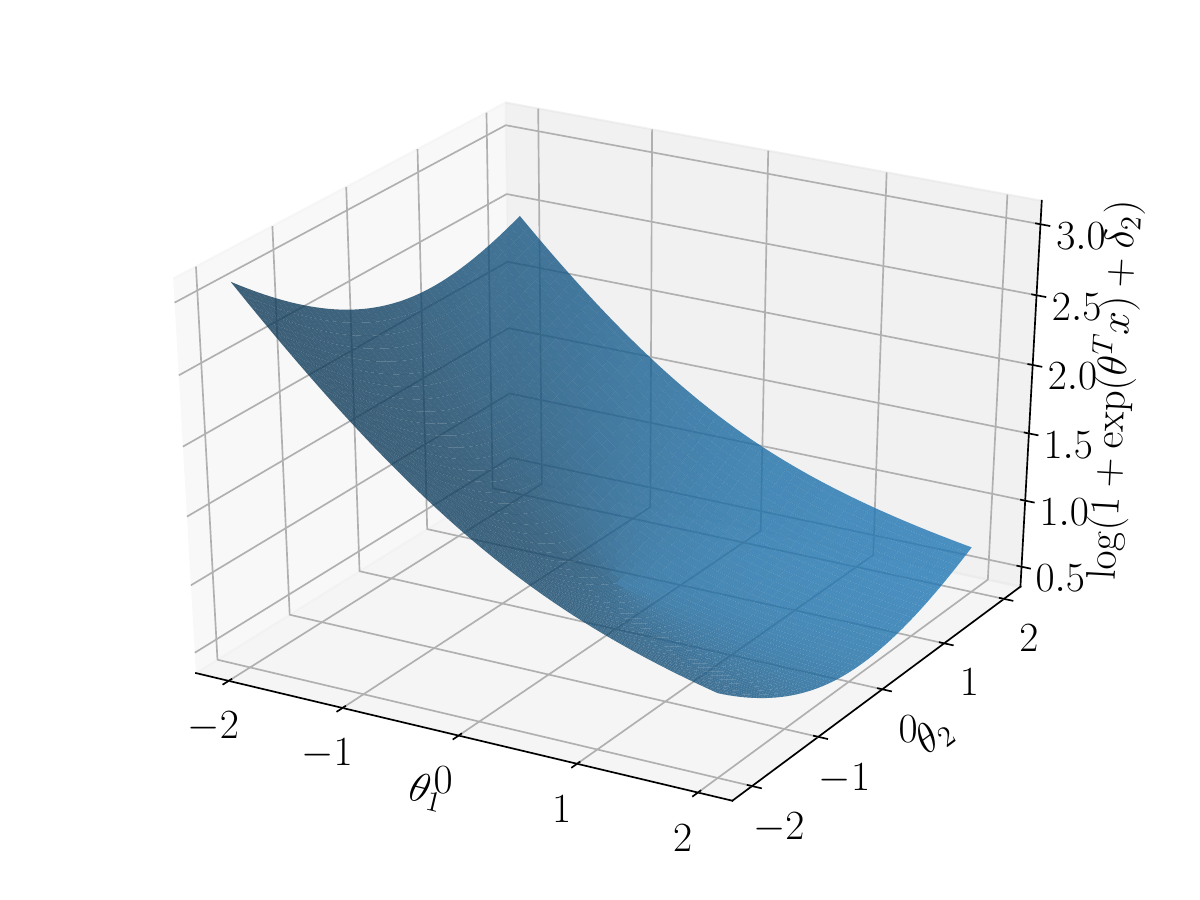}}
        \subfloat[][Adversarial loss with $L_\infty$]{\includegraphics[width=0.325\linewidth]{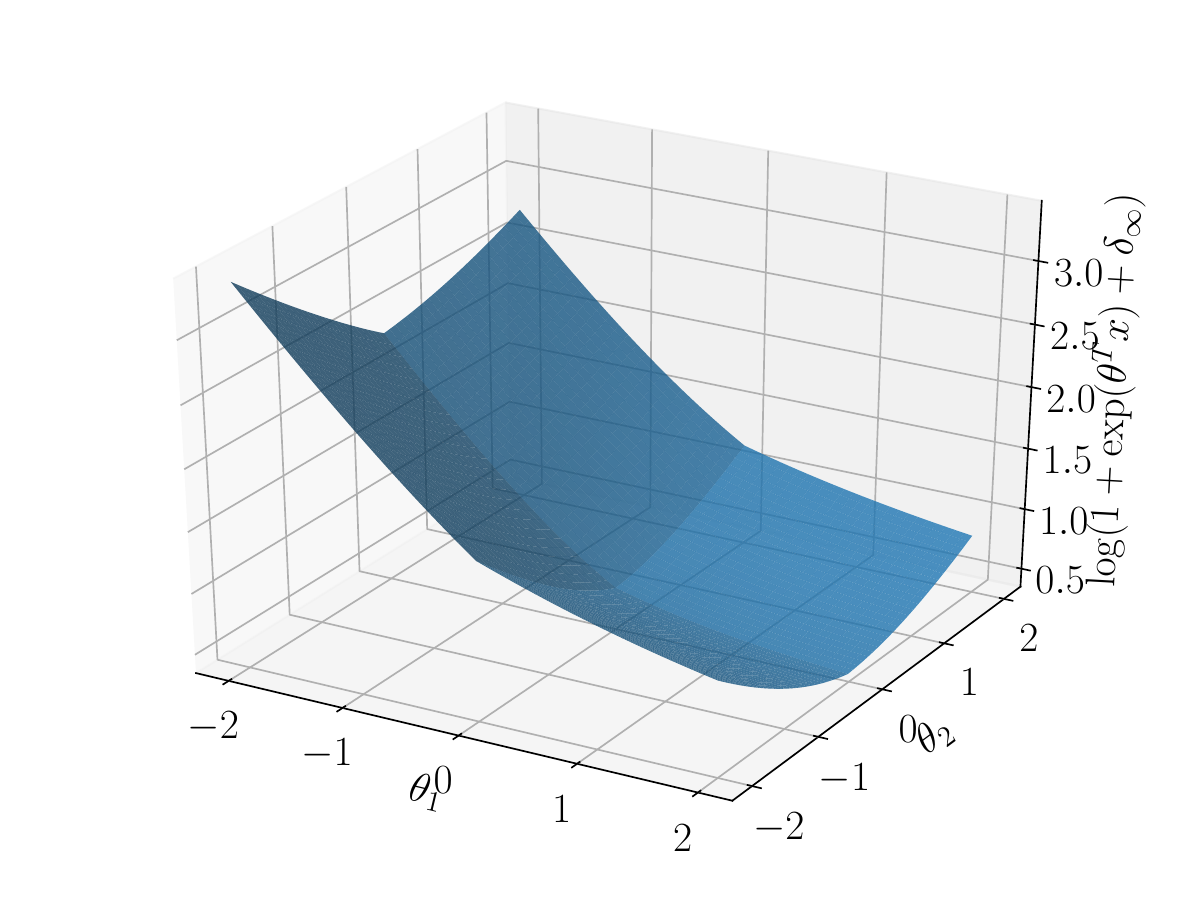}}
        \caption{Loss surface for $L(\bm{\theta})=1/N\sum_n \log(1+\exp(-y_n\bm{\theta}^T(\bm{x}_n+\bm{\delta})))$ where $\bm{x}_n \sim \mathcal{N}(\bm{0},\bm{I}_2)$, $y_n=\mathrm{sign}(x_{n,1})$, $N=100$, and 
        $\theta_{i}\in[-2,2]$ for $i=1,2$. Adversarial examples are constrained as $||\bm{\delta}||_p\leq 0.6$.
        Clean loss (a) has the Lipschitz continuous gradient while the gradient of adversarial loss with $L_2$ 
        (b) is not smooth at $\bm{\theta}=\bm{0}$, and the gradient of adversarial loss with $L_\infty$ 
        (c) is not smooth at $\theta_i=0$.}
        \label{Plot2}
    \end{figure*}
    We applied EntropySGD to the adversarial loss functions in Fig.~\ref{Plot2}.
    We computed the integral in the loss functions of EntropySGD by using scipy.
    Figure~\ref{PlotEnSGD2} plots the loss surface of EntropySGD in $\theta_*\in [-2,2]$.
    Figures~\ref{PlotEnSGD2} (a) and (b) show the loss functions $-F(\bm{\theta})$ when $L(\bm{\theta})=1/N\sum_i \ell (\bm{x}_i+\bm{\delta},\bm{\theta}_i,y)$
    where $\bm{x}_i\sim \mathcal{N}(\bm{0},\bm{I}_2)$ and $y_i=\mathrm{sign}(x_{i,1})$.
    In contrast to Fig.~\ref{Plot2}, loss functions for EntropySGD are smooth.
    Especially, even though the smoothnesses of $L(\bm{\theta})$ for $L_2$ and $L_\infty$ adversarial attacks are different,
    smoothnesses of their loss functions $-F(\bm{\theta})$ are almost the same.
    Note that the optimal points of $-F(\bm{\theta})$ are the same as the original loss functions $L(\bm{\theta})$.
    Thus, EntropySGD smoothens non-smooth functions as Theorem~\ref{EnThm}.
    \begin{figure*}[tb]
        \centering
        \subfloat[][$L_2$ adversarial attacks]{\includegraphics[width=0.24\linewidth]{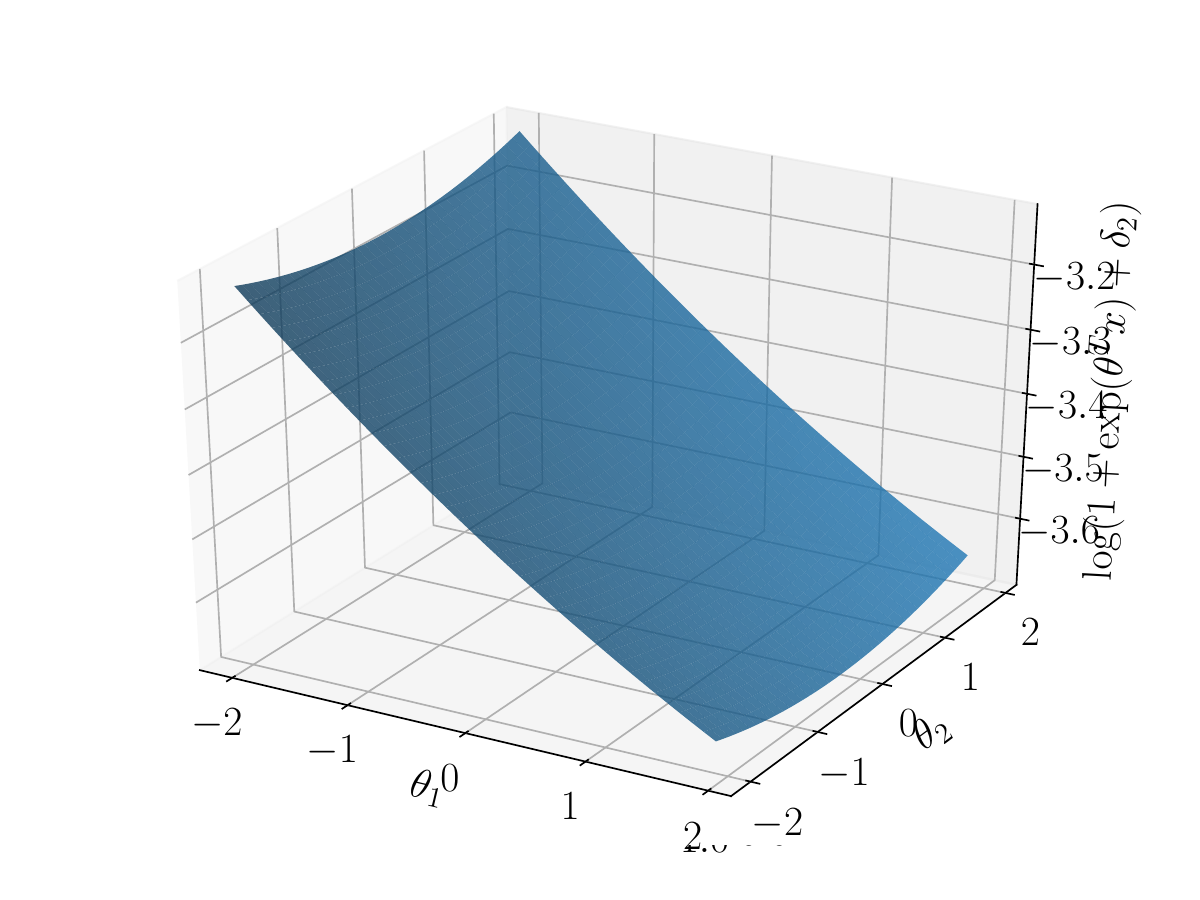}}
        \subfloat[][$L_\infty$ adversarial attacks]{\includegraphics[width=0.24\linewidth]{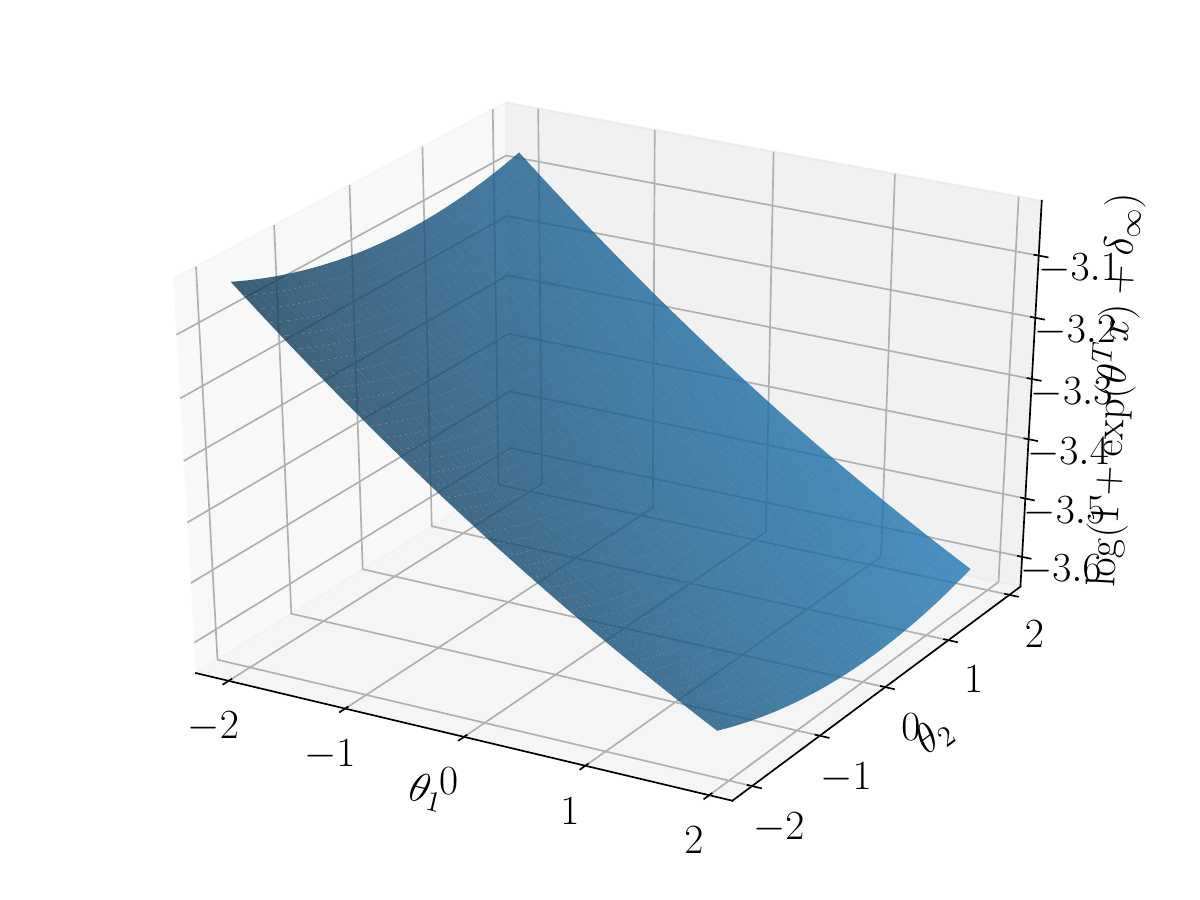}}
        \caption{We apply EntropySGD to problems in Fig.~\ref{Plot2}.
        (a) and (b) correspond to Fig.~\ref{Plot2} (b) and (c) respectively.}
        \label{PlotEnSGD2}
    \end{figure*}
    \section{Extension for the second order method}\label{2ndEnSGDsec}
    We show that the Hessian matrix of EntropySGD is composed of a variance-covariance matrix $\bm{\Sigma}_{\bm{\theta}^{\prime}}$
    of $p_{\bm{\theta}}(\bm{\theta}^{\prime})$.
    Since we can estimate the variance-covariance matrix by using stochastic gradient Langevin dynamics (SGLD),
    it is easy to extend EntropySGD into Newton method-based EntropySGD.
    Since the computation of the inverse matrix of the Hessian matrix is high,
    we assume that the variance-covariance matrix is a diagonal matrix composed of
    only the variance $\bm{\Sigma}_{\bm{\theta}^\prime}\!\approx\!\mathrm{diag}(\bm{\sigma}_{\bm{\theta}^\prime})$
    where $\mathrm{diag}(\bm{v})$ is a diagonal matrix whose $i$-th diagonal element is $v_i$.
    Then, we have $(-\nabla^2F)\!=\!\mathrm{diag}(\gamma\bm{1}-\gamma^2\bm{\sigma}_{\bm{\theta}^\prime})$
     and derive the following update:
    \begin{align}
        \textstyle
        \bm{\theta}&\textstyle\leftarrow\bm{\theta}-\eta\gamma \bm{h}\odot(\bm{\theta}-\bar{\bm{\theta}})\\
        \textstyle
        h_j&=\textstyle\frac{1}{\gamma -\gamma^2\bar{\sigma}_{j,\bm{\theta}^{\prime}}}\nonumber
    \end{align}
    where $\odot$ is an element-wise product.
     $\bar{\sigma}_{j,\bm{\theta}^{\prime}}$ is the estimated variance of the $j$-th parameter $\theta_j^{\prime}$ following $p_{\bm{\theta}}(\bm{\theta}^{\prime})$.
    This update rule corresponds to the Newton method under the diagonal approximation of $\bm{\Sigma}_{\bm{\theta}^\prime}$:
    i.e., element-wise product of $\bm{h}$ corresponds to the matrix product of the inverse Hessian matrix.
    \subsection{Algorithm}
Algorithm~\ref{Alg} describes the training procedure for adversarial training with second order EntropySGD.
We also show the combination of AWP in this algorithm.
$\bar{\bm{\xi}}$ is estimated $\mathbb{E}_{p(\bm{\theta}^\prime)}[\bm{\theta}^{\prime}\odot \bm{\theta}^{\prime}]$
by using SGLD, and the variance is computed at line 16 as $\bar{\xi}_j-\bar{\theta}_j^2$.
We just compute PGD attacks and AWP perturbations in the loop of EntropySGD.
Note that the computation cost of EntropySGD and 2nd order EntropySGD are almost the same as SGD.
\begin{algorithm}[tb]
    \caption{AT (+AWP) + 2nd order EntropySGD}
    \label{Alg}
    \begin{algorithmic}[1]
        \REQUIRE Current weights $\bm{\theta}$, Langevin iterations $L$
        \STATE Hyperparameters: scope $\gamma$, learning rate $\eta$, SGLD step size $\eta '$
        \STATE $\bm{\theta}^{\prime}, \bar{\bm{\theta}}\leftarrow \bm{\theta}$
        \STATE $\bar{\bm{\xi}} \leftarrow \bm{\theta}\odot \bm{\theta}$
        \FOR{$l \leq L$}
        \STATE Sample a minibatch $D_{l}=\{(\bm{x}_i,y_i) \}_{i=1}^{|D_l|}$
        \FOR{$\tau \leq T$}
            \STATE $\bm{\delta }_i \leftarrow \Pi_{\varepsilon}\left(\bm{\delta}_i+\eta_{\mathrm{P}}\mathrm{sgn}\left(\nabla_{\bm{\delta}} \ell\left(\bm{x}_i+\bm{\delta}, y_i, \bm{\theta}^{\prime}\right)\right)\right)$ for $i=1,\dots,|D_l|$
        \ENDFOR
        \STATE Update $\bm{v}\in \mathcal{V}$ to maximize $\frac{1}{|D_l|}\sum_i \ell(\bm{x}_i+\bm{\delta}_i,y_i,\bm{\theta}^{\prime}+\bm{v})$ \cite{awp} if using AWP
        \STATE $d\bm{\theta}^{\prime} \leftarrow -\frac{1}{|D_l|}\sum_i \nabla_{\bm{\theta}'}\ell(\bm{x}_i+\bm{\delta}_i,y_i,\bm{\theta}'+\bm{v})-\gamma(\bm{\theta}-\bm{\theta}^{\prime})$
        \STATE $\bm{\theta}^{\prime}\leftarrow \bm{\theta}^{\prime}-\eta ' d\bm{\theta}^{\prime}+\sqrt{\eta '}\varepsilon_E \mathcal{N}(\bm{0},\bm{I})$
        \STATE $\bar{\bm{\theta}}\leftarrow (1-\alpha)\bar{\bm{\theta}}+\alpha \bm{\theta}^{\prime}$
        \STATE $\bar{\bm{\xi}}\leftarrow (1-\alpha)\bar{\bm{\xi}}+\alpha \bm{\theta}^{\prime}\odot \bm{\theta}^{\prime}$
    \ENDFOR
    \FOR{$j\leq d$}
    \STATE $h_j=\frac{1}{\gamma-\gamma^2 (\bar{\xi}_j-\bar{\theta}_j^2)}$
    \ENDFOR
    \STATE $\bm{\theta}\leftarrow \bm{\theta}-\eta \gamma\bm{h}\odot(\bm{\theta}-\bar{\bm{\theta}})$
    \end{algorithmic}
\end{algorithm}
\section{Experimental setup}
This section gives the experimental conditions.
Our experimental codes are based on source codes provided by \citet{awp},
and our implementations of EntropySGD are based on the codes provided by \citet{chaudhari2019entropy}.
Datasets of the experiments were CIFAR10, CIFAR100 \cite{cifar}, and SVHN \cite{svhn}.
We compared the convergences of adversarial training when using SGD and EntropySGD.
In addition, we evaluate the combination of EntropySGD and AWP \cite{awp},
which injects adversarial noise into the parameter to flatten the loss landscape.

\subsection{CIFAR10\cite{cifar}}
We used ResNet-18 (RN18) \cite{resnet} and WideResNet-34-10 (WRN) \cite{WRN} following \cite{awp}.
We used untargeted projected gradient descent (PGD), which is the most popular white box attack.
The hyperparameters for PGD were based on \cite{awp}.
The $L_\infty$ norm of the perturbation $\varepsilon\!=\!8/255$ at training time.
For PGD, we randomly initialized the perturbation and updated for 10 iterations with a step size of 2/255
at training time for CIFAR10.
At evaluation time, we use AutoAttack \cite{AutoAttack} for CIFAR10.
In addition, we used PGD with 20 iterations and a step size of 2/255 for CIFAR10
for visualization of loss land scape and the evaluation of the convergences of adversarial training.
For AT+EntropySGD, we set $\gamma=0.03$, $\varepsilon_{E}=1\times10^{-4}$, $\eta=0.1$, $\eta^\prime=0.1$, and 
$L=20$ for RN18, and $L=30$ for WRN.
Note that we coarsely tuned $\gamma$ and $\eta^\prime$ and found that the effect of
these parameters tuning is less than that of $L$, and thus, we use the settings of \cite{chaudhari2019entropy} for these parameters. 
For AWP+EntropySGD, we set $\gamma=0.03$, $\varepsilon_{E}=1\times10^{-4}$, $\eta=0.1$, $\eta^\prime=0.1$, and 
$L=30$, and $\gamma_{A}=0.005$ of AWP.
For AWP, we set $\gamma_{A}=0.01$ following \cite{awp}.
Note that we found that AWP with $\gamma_{A}=0.005$ does not outperform AWP with $\gamma_{A}=0.01$
when using SGD.
Since the training accuracy of AWP contains the effect of the adversarial weight perturbation,
improvement of convergence of training accuracy of AWP ($\gamma_{A}=0.005$) + EntropySGD
might be caused by the small weight perturbation.
Thus, we plot AWP ($\gamma_{A}=0.005$) with SGD in \rfig{TrainAcc2}, in which other results are the same as those in Figure~1 of the main paper.
This figure shows that AWP ($\gamma_{A}=0.005$) + EntropySGD outperforms AWP ($\gamma_{A}=0.005$) with SGD in terms of the convergence.
In addition, the test accuracy of AWP ($\gamma_{A}=0.005$) with SGD is lower than that of AWP ($\gamma_{A}=0.01$), which is plotted in Figure~1 of the main paper.
\begin{figure}[tb]
    \centering
    \includegraphics[width=.5\linewidth]{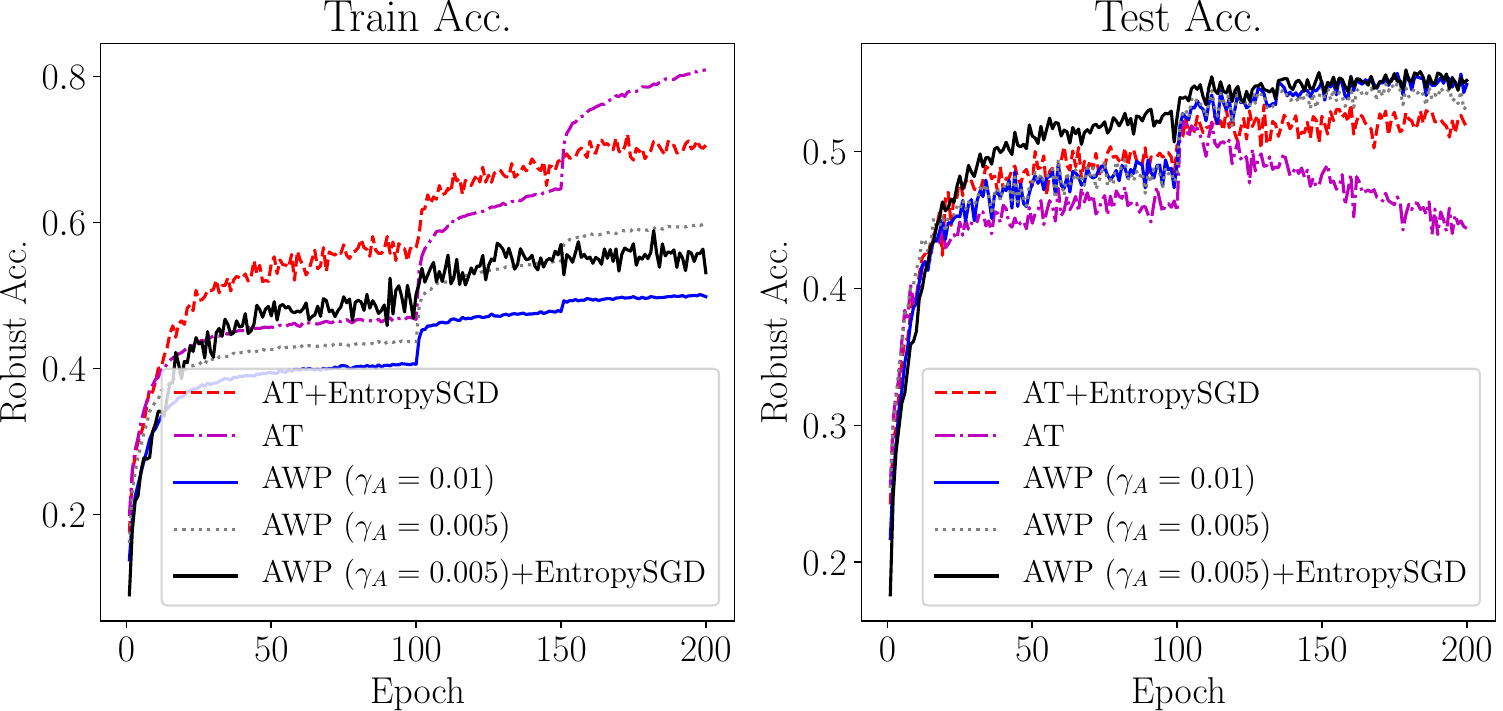}
    \caption{Robust Accuracy against PGD vs Epochs. Left figure is training accuracy, and right figure is test accuracy.
    AT and AT+EntropySGD denote adversarial training using SGD and adversarial training using EntropySGD, respectively.
    AWP and AWP+EntropySGD denote adversarial training with AWP using SGD and adversarial training using EntropySGD, respectively.
    Note that we use PGD with 10 iterations for training accuracy and PGD with 20 iterations for test accuracy.
    }
    \label{TrainAcc2}
\end{figure}

For the preprocessing, we standardized data by using the mean of [0.4914, 0.4822, 0.4465],
and standard deviations of [0.2471, 0.2435, 0.2616].
The gradient of the preprocessing is considered in the generation of PGD.

The learning rates of SGD and EntropySGD are set to 0.1 and divided by 10 at the 100-th and 150-th epoch,
and we used early stopping by evaluating test accuracies.
We used momentum of 0.9 and weight decay of 0.0005.
We trained the models three times and show the average and standard deviation of test accuracies.
For evaluating test accuracies against runtime,
we used NVIDIA Tesla V100 SXM2 32GB GPUs and Intel(R) Xeon(R) Silver 4110 CPU @ 2.10GHz.
\subsection{CIFAR100}
We used RN18 \cite{resnet} and untargeted projected PGD.
The hyper-parameters for PGD were based on \cite{awp}.
The $L_\infty$ norm of the perturbation $\varepsilon\!=\!8/255$ at training time.
For PGD, we randomly initialized the perturbation and updated it for 10 iterations with a step size of 2/255
at training time for CIFAR100.
At evaluation time, we use PGD with 20 iterations and a step size of 2/255 for CIFAR100.
For EntropySGD, we set $\gamma=0.03$,
$\varepsilon_{E}=1\times10^{-4}$, $\eta=0.1$, $\eta^\prime=0.1$, and $L=30$.
The learning rates of SGD and EntropySGD are set to 0.1 and divided by 10 at the 100-th and 150-th epoch,
and we used early stopping by evaluating test accuracies.
We used momentum of 0.9 and weight decay of 0.0005.
For the preprocessing,
we standardized data by using the mean of 
[0.5070751592371323, 0.48654887331495095, 0.4409178433670343],
and standard deviations of [0.2673342858792401, 0.2564384629170883, 0.27615047132568404].
The gradient of the preprocessing is considered in the generation of PGD.
The hyperparameter $\gamma_A$ of AWP is set in 0.01 for AWP, and  $\gamma_A$ of AWP is set in 0.007 for AWP+EntropySGD.
We trained models for three times and
show the average and standard deviation of test accuracies.
\subsection{SVHN~\cite{svhn}}
We used RN18 and untargeted PGD.
The hyperparameters for PGD were based on \cite{awp}.
The $L_\infty$ norm of the perturbation $\varepsilon\!=\!8/255$ at training time.
For PGD, we randomly initialized the perturbation and updated it for 10 iterations with a step size of 1/255
at training time for SVHN.
At evaluation time, we use PGD with 20 iterations and a step size of 1/255 for SVHN.
For training of SVHN, we did not apply EntropySGD and AWP for 5 epochs following~\cite{awp}.
For EntropySGD, we set $\gamma=0.03$, $\varepsilon_{E}=1\times10^{-4}$, $\eta=0.1$, $\eta^\prime=0.1$, and
$L=30$.
For the preprocessing,
we standardized data by using the mean of 
[0.5, 0.5, 0.5],
and standard deviations of [0.5, 0.5, 0.5].
The gradient of the preprocessing is considered in the generation of PGD.
The learning rates of SGD and EntropySGD are set to 0.01 and divided by 10 at the 100-th and 150-th epoch,
and we used early stopping by evaluating test accuracies.
We used momentum of 0.9 and weight decay of 0.0005.
The hyperparameter $\gamma_A$ of AWP is set to 0.01 for AWP, and 0.005 for AWP+EntropySGD.
We trained models three times and
show the average and standard deviation of test accuracies.
\section{Additional experimental results}
\subsection{Visualization of loss landscape}
To investigate the cause of improvements of EntropySGD,
we visualize the loss landscape in parameter space. 
Figure~\ref{LossSurf} shows adversarial loss against perturbations of the parameter by using Filter Normalization \cite{VisW}
 $\bm{w}+\alpha \frac{\bm{d}}{||\bm{d}||}||\bm{w}|| $ following \cite{awp}.
 In this figure, we use the models that achieves the most test robust accuracy.
 We show loss landscapes for training data and test data.
 Since we use early stopping for all methods, no loss landscapes are very sharp, which coincides with the results of early stopping in \cite{awp}. 
 This figure shows that EntropySGD does not necessarily flatten loss landscapes while training loss of EntropySGD is smaller than those of other methods. 
 Compared with EntropySGD, AT and AWP lead to a slight underfit to achieve good generalization performance.
Therefore, EntropySGD achieves a good trade-off between training accuracy and test accuracy.
Note that we show loss of $L_{\varepsilon}(\bm{\theta})$, not $F(\bm{\theta})$ for EntropySGD,
and thus, this figure does not evaluate the smoothness of $F(\bm{\theta})$.
We could not evaluate $F(\bm{\theta})$ because the computation cost of $F(\bm{\theta})$ is too high for DNNs.
\begin{figure}[tb]
    \centering
    \begin{tabular}{cc}
    \subfloat[][ResNet18]{
    \includegraphics[width=.4\linewidth]{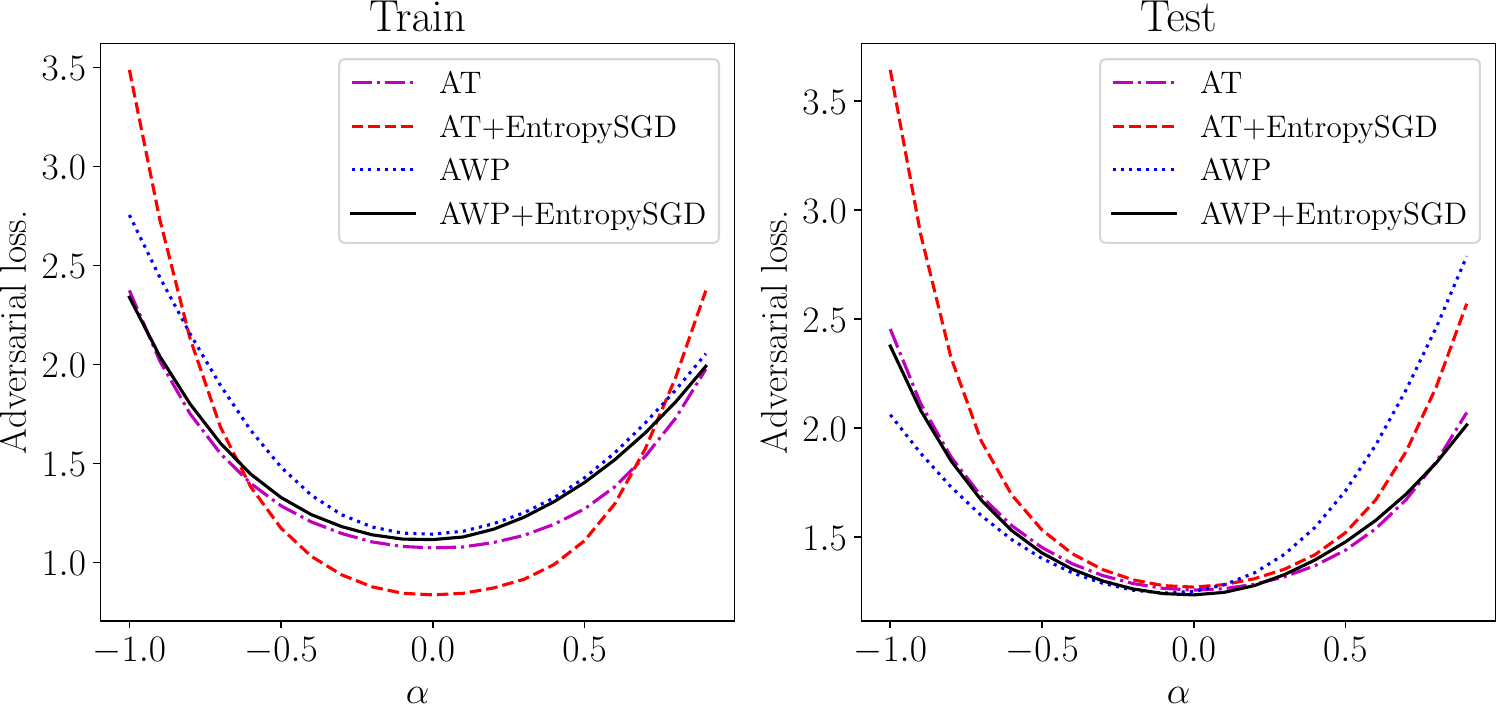}
    }
    \centering
    \subfloat[][WideResNet]{
        \includegraphics[width=.4\linewidth]{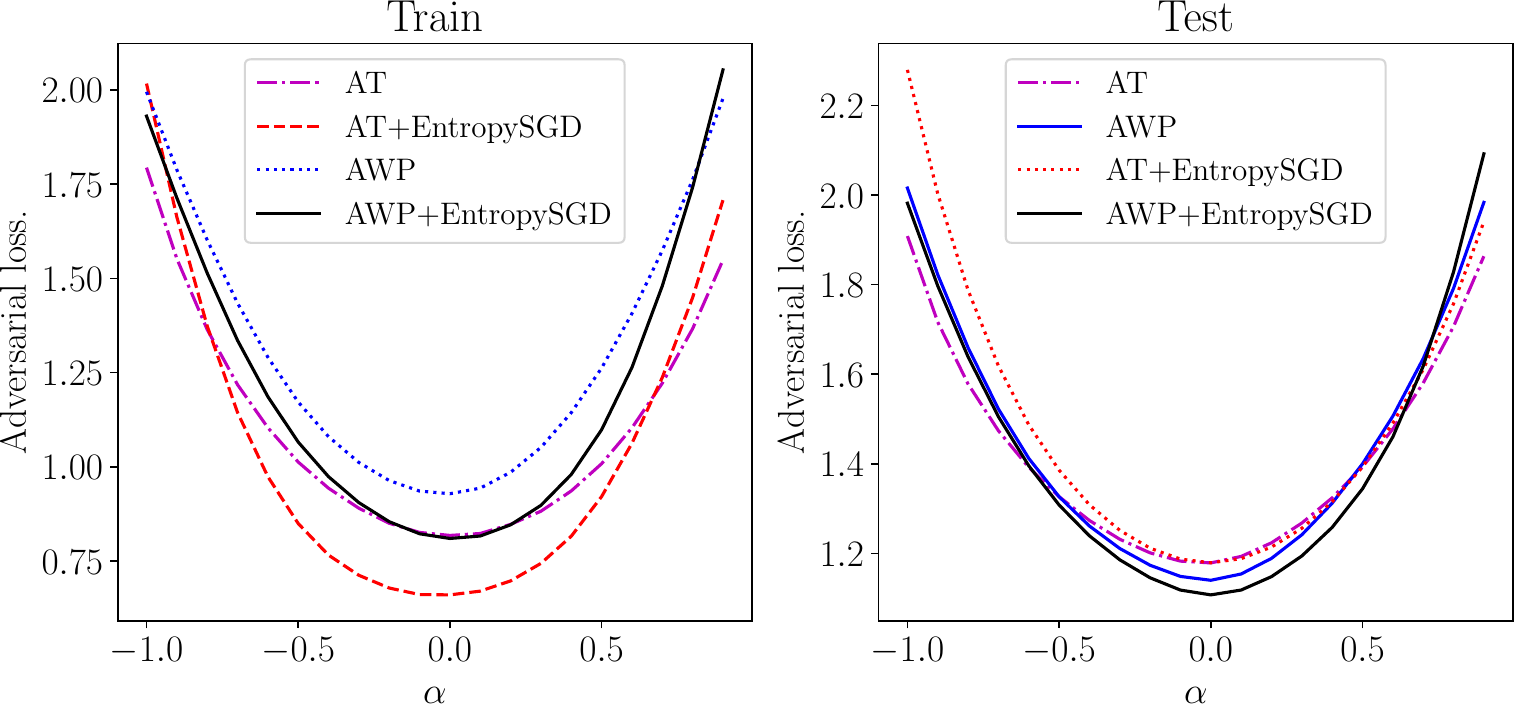}
    }
    \end{tabular}
    \caption{Visualizing adversarial loss $L_{\varepsilon}$ against perturbation of parameters by using Filter Normalization.
    The size of perturbation $\alpha$ is changed in from -1 to 1 in increments of 0.1.}
    \label{LossSurf}
\end{figure}
\subsection{Evaluation with TRADES}
\subsubsection{Setup}
We evaluated the performance of EnSGD with TRADES.
We trained models by using TRADES with SGD (TRADES), TRADES with EnSGD (+EnSGD),
TRADES with AWP and SGD (+AWP), TRADES with AWP and EnSGD (AWP+EnSGD).
Since we observed that training of EnSGD with AWP becomes slow after the 150-th epoch,
we additionally evaluated models trained by switching the optimization method from EnSGD to SGD at the 150-th epoch (+AWP+EnSGD(tuned)):
we trained models by using EnSGD for 150 epochs and trained them by using SGD after the 150-th epoch.
In this experiment, we followed the setup in the provided code of \cite{awp}.
For training of TRADES+AWP, TRADES+AWP+EnSGD, and TRADES+AWP+EnSGD(tuned)), 
we did not apply EntropySGD and AWP for 10 epochs following~\cite{awp}.
We used WideResNet-34-10 (WRN) \cite{WRN} and CIFAR10.
The hyperparameters for TRADES were based on \cite{awp}.
The $L_\infty$ norm of the perturbation is $\varepsilon\!=\!0.031$ at training time.
We randomly initialized the perturbation and updated for 10 iterations with a step size of 0.003 at training time.
At evaluation time, we use AutoAttack \cite{AutoAttack}.
For EnSGD, we set $\gamma=0.03$, $\varepsilon_{E}=1\times10^{-4}$, $\eta=0.1$, $\eta^\prime=0.1$, and 
 $L$ in EnSGD is set to 20 for +EnSGD and 30 for +AWP+EnSGD. 
 Additionally, we set $\alpha$ to 0.3 in +AWP+EnSGD, while we set $\alpha$ to 0.75 in other experiments including +AWP+EnSGD(tuned). 
For AWP, we set $\gamma_{A}=0.005$ following \cite{awp}.
The learning rates of SGD and EnSGD are set to 0.1 and divided by 10 at the 100-th and 150-th epoch,
and we used early stopping by evaluating test robust accuracies against untargeted PGD of 20 iterations.
We used momentum of 0.9 and weight decay of 0.0005.
We trained models for three times and show the average and standard deviation of test accuracies.
We normalized data but did not standardize data following \cite{awp}.
\subsubsection{Results}
\rtab{TRADEStab} lists the average of test robust accuracies when using TRADES with EnSGD.
TRADES with EnSGD (+EnSGD) outperforms TRADES with SGD (TRADES),
and TRADES with AWP and EnSGD (+AWP+EnSGD) outperforms TRADES with AWP and SGD (+ AWP).
Furthermore, when we switched the training method from EnSGD to SGD after 150 epochs (+AWP+EnSGD(tuned)),
models achieved the highest robust accuracies against AutoAttack.
This result indicates that non-smoothness particularly affects training at the early stage of training in which the learning rate is large,
and the performance at the early stage of training affects the final result of training.
Note that switching optimization methods are sometimes used for improving generalization performance \cite{keskar2017improving}.
\begin{table}[tbp]
        \centering
        \caption{Robust accuracies against AutoAttack on CIFAR10 using TRADES.}
        \label{TRADEStab}
        \begin{tabular}{cccccccccc}\toprule
        &TRADES&+EnSGD&+AWP&+AWP+EnSGD&+AWP+EnSGD(tuned)\\\midrule
        WRN&$52.9\!\pm\! 0.3$&$53.4\pm 0.1$&$55.6\!\pm\! 0.4$&$55.8\!\pm \!0.3$&$\bm{56.3\!\pm \!0.1}$\\
        \bottomrule
        \end{tabular}
    \end{table}
\subsection{Evaluation of second order entropySGD}
In Section~\ref{2ndEnSGDsec}, we present EntropySGD using variance-covariance matrix (2ndEnSGD).
We also conducted the experiments for 2ndEnSGD whose setup is the same as that for EnSGD.
\rtab{AAtab2nd} lists the test robust accuracies, and we can see that 2ndEnSGD performs almost the same as EnSGD.
This might be because the diagonal approximation in 2ndEnSGD is not very effective.
\begin{table}[tbp]
        \centering
    
        \caption{Robust accuracies against AutoAttack on CIFAR10. EnSGD denotes EntropySGD, and Robust accuracies against PGD on CIFAR100 and SVHN. EnSGD denotes EntropySGD.}
        \label{AAtab2nd}
        \resizebox{\columnwidth}{!}{
        \begin{tabular}{cccccccc}\toprule
        &AT&AT+EnSGD&AT+2ndEnSGD&AWP&AWP+EnSGD&AWP+2ndEnSGD\\\midrule
        CIFAR10(AA, RN18)&$48.0\!\pm\! 0.2$&$49.2\!\pm\! 0.4$&$48.93\!\pm\! 0.09$&$49.9 \!\pm\! 0.4$&$\bm{50.5\!\pm\! 0.2}$&$50.4 \pm 0.3$\\
        CIFAR10(AA, WRN)&$51.9\!\pm\! 0.5$&$53.3\!\pm\! 0.3$&$53.5 \pm 0.2$&$53.2\!\pm\! 0.3$&$54.72\!\pm \!0.05$&$\bm{54.9 \pm 0.2}$\\\midrule
        SVHN (PGD, RN18)&$53.1\!\pm\! 0.6$&$59.1\!\pm\! 0.4$&$59.3\!\pm\! 0.1$&$59.3\!\pm\! 0.1$&$\bm{59.91\!\pm\!0.09}$&$\bm{59.9\!\pm\! 0.1}$\\
        CIFAR100 (PGD, RN18)&$27.66\!\pm\! 0.04$&$28.69\!\pm\! 0.03$&$28.94\!\pm\! 0.05$&$\bm{30.95\!\pm\! 0.07}$&$30.9 \!\pm\! 0.1$&$30.7\!\pm\! 0.1$\\
        \bottomrule
        \end{tabular}
        }
    \end{table}
\end{document}